\documentclass{article}
\usepackage[accepted]{icml2019}
\usepackage{hyperref}
\usepackage[T1]{fontenc}
\usepackage{microtype}
\usepackage{graphicx}
\usepackage{subfigure}
\usepackage{booktabs}
\usepackage{amssymb,amsmath,amsthm}
\usepackage{mathtools}

\newtheorem{theorem}{Theorem}

\newtheorem{corollary}[theorem]{Corollary}

\newtheorem{proposition}[theorem]{Proposition}
\newtheorem{lemma}[theorem]{Lemma}

\newtheoremstyle{named}{}{}{\itshape}{}{\bfseries}{.}{.5em}{\thmnote{#3 }#1} \theoremstyle{named}

\newcommand{\R}{{\mathbb R}}

\newcommand{\E}[1]{{\mathbb E}\left [#1\right]}
\renewcommand{\P}{{\mathbb P}}
\newcommand{\ep}{\varepsilon}

\newcommand{\mB}{\mathcal B}
\newcommand{\mN}{\mathcal N}

\newcommand{\gives}{\ensuremath{\rightarrow}}
\newcommand{\x}{\ensuremath{\times}}

\newcommand{\mO}{\mathcal O}
\newcommand{\nin}{n_{\mathrm{in}}}

\newcommand{\abs}[1]{\ensuremath{\left| #1 \right|}}
\newcommand{\lr}[1]{\ensuremath{\left(#1 \right)}}
\newcommand{\norm}[1]{\left\lVert#1\right\rVert}
\newcommand{\inprod}[2]{\ensuremath{\left\langle#1,#2\right\rangle}}

\newcommand{\twiddle}[1]{\ensuremath{\widetilde{#1}}}

\newcommand{\w}{\omega}

\newcommand{\set}[1]{\ensuremath{\{#1\}}}

\def\XXint#1#2#3{{\setbox0=\hbox{$#1{#2#3}{\int}$} \vcenter{\hbox{$#2#3$}}\kern-.5\wd0}}

\DeclareMathOperator{\Relu}{ReLU}

\DeclareMathOperator{\vol}{vol}
\DeclareMathOperator{\Var}{Var}

\icmltitlerunning{Complexity of Linear Regions in Deep Networks}

\begin{document}
\twocolumn[
\icmltitle{Complexity of Linear Regions in Deep Networks}
\icmlsetsymbol{equal}{*}

\begin{icmlauthorlist}
\icmlauthor{Boris Hanin}{equal,tamu}
\icmlauthor{David Rolnick}{equal,upenn}
\end{icmlauthorlist}

\icmlaffiliation{tamu}{Department of Mathematics, Texas A\&M University and Facebook AI Research, New York}
\icmlaffiliation{upenn}{University of Pennsylvania}

\icmlcorrespondingauthor{Boris Hanin}{bhanin@tamu.edu}
\icmlcorrespondingauthor{David Rolnick}{drolnick@seas.upenn.edu}

\icmlkeywords{ReLU, linear region, expressivity, learnability, initialization, random network}
\vskip 0.3in
]

\printAffiliationsAndNotice{\icmlEqualContribution}
\begin{abstract}
It is well-known that the expressivity of a neural network depends on its architecture, with deeper networks expressing more complex functions. In the case of networks that compute piecewise linear functions, such as those with ReLU activation, the number of distinct linear regions is a natural measure of expressivity. It is possible to construct networks with merely a single region, or for which the number of linear regions grows exponentially with depth; it is not clear where within this range most networks fall in practice, either before or after training. In this paper, we provide a mathematical framework to count the number of linear regions of a piecewise linear network and measure the volume of the boundaries between these regions. In particular, we prove that for networks at initialization, the average number of regions along any one-dimensional subspace grows linearly in the total number of neurons, far below the exponential upper bound. We also find that the average distance to the nearest region boundary at initialization scales like the inverse of the number of neurons. Our theory suggests that, even after training, the number of linear regions is far below exponential, an intuition that matches our empirical observations. We conclude that the practical expressivity of neural networks is likely far below that of the theoretical maximum, and that this gap can be quantified.
\end{abstract}
\begin{figure}
  \centering
\includegraphics[scale=0.35]{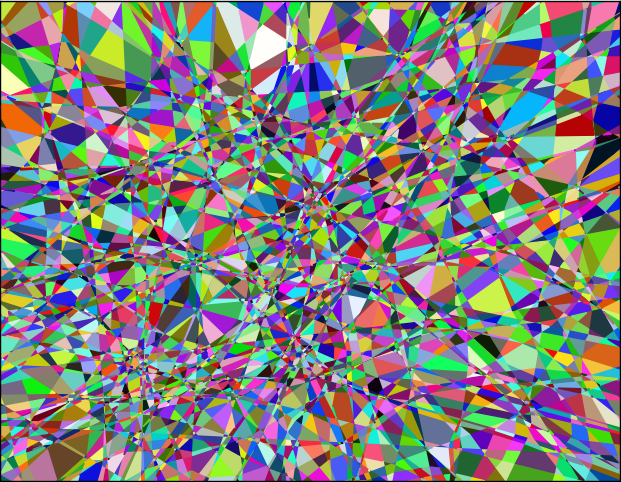}
\caption{How many linear regions? This figure shows a two-dimensional slice through the 784-dimensional input space of vectorized MNIST, as represented by a fully-connected ReLU network with three hidden layers of width 64 each. Colors denote different linear regions of the piecewise linear network.}
\label{fig:2d_regions_init}
\vskip-.2in
\end{figure}

\section{Introduction}
A growing field of theory has sought to explain the broad success of deep neural networks via a mathematical characterization of the ability of these networks to approximate different functions of input data. Many such works consider the \emph{expressivity} of neural networks, showing that certain functions are more efficiently expressible by deep architectures than by shallow ones (e.g.~\citet{bianchini2014complexity, montufar2014number, telgarsky2015representation, lin2017does, rolnick2017power}). It has, however, also been noted that many expressible functions are not efficiently \emph{learnable}, at least by gradient descent \citep{shalev2017failures}. More generally, the \emph{typical} behavior of a network used in practice, the \emph{practical expressivity}, may be very different from what is theoretically attainable. To adequately explain the power of deep learning, it is necessary to consider networks with parameters as they will naturally occur before, during, and after training.

Networks with a piecewise linear activation (e.g.~ReLU, hard $\tanh$) compute piecewise linear functions for which input space is divided into pieces, with the network computing a single linear function on each piece (see Figures \ref{fig:2d_regions_init}-\ref{fig:1d_regions}). Figure \ref{fig:evolution_of_regions} shows how the complexity of these pieces, which we refer to as \emph{linear regions}, changes in a deep ReLU net with two-dimensional inputs. Each neuron in the first layer splits the input space into two pieces along a hyperplane, fitting a different linear function to each of the pieces. Subsequent layers split the regions of the preceding layers. The local density of linear regions serves as a convenient proxy for the local complexity or smoothness of the network, with the ability to interpolate a complex data distribution seeming to require fitting many relatively small regions.
The topic of counting linear regions is taken up by a number of authors \citep{telgarsky2015representation, montufar2014number, serra2017bounding, raghu2017expressive}. 

A worst case estimate is that every neuron in each new layer splits each of the regions present at the previous layer, giving a number of regions exponential in the depth. Indeed this is possible, as examined extensively e.g.~in \citet{montufar2014number}. An example of \citet{telgarsky2015representation} shows that a sawtooth function with $2^n$ teeth can be expressed exactly using only $3n+4$ neurons, as shown in Figure \ref{fig:sawtooth}. However, even slightly perturbing this network (by adding noise to the weights and biases) ruins this beautiful structure and severely reduces the number of linear pieces, raising the question of whether typical neural networks actually achieve the theoretical bounds for numbers of linear regions.

\begin{figure}[htb]
    \centering
    \includegraphics{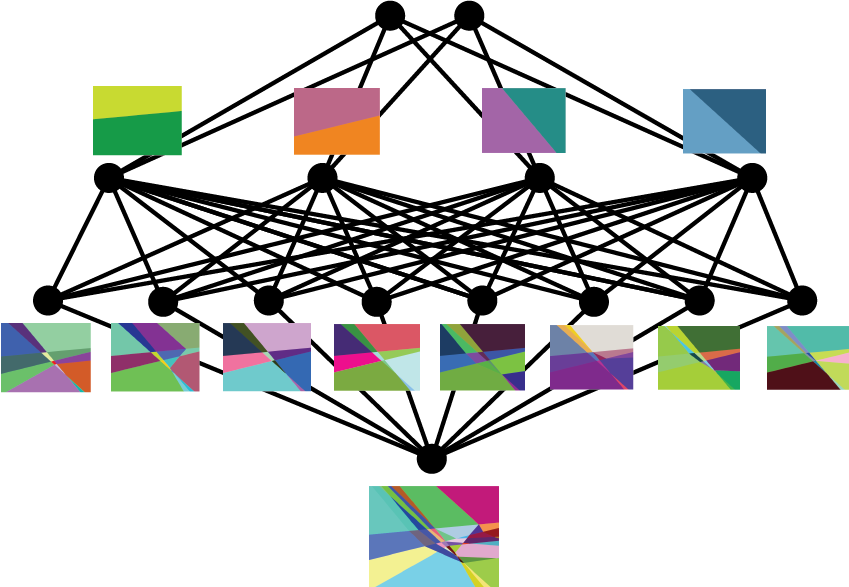}
    \caption{Evolution of linear regions within a ReLU network for 2-dimensional input. Each neuron in the first layer defines a linear boundary that partitions the input space into two regions. Neurons in the second layer combine and split these linear boundaries into higher level patterns of regions, and so on. Ultimately, the input space is partitioned into a number of regions, on each of which the neural network is given by a (different) linear function. During training, both the partition into regions and the linear functions on them are learned.}
    \label{fig:evolution_of_regions}
\end{figure}

Figure \ref{fig:2d_regions_init} also invites  measures of complexity for piecewise linear networks beyond region counting. The boundary between two linear regions can be straight or can be bent in complex ways, for example, suggesting the \textit{volume of the boundary} between linear regions as complexity measure for the resulting partition of input space. In the 2D example of Figure \ref{fig:2d_regions_init}, this corresponds to computing perimeters of the linear pieces. As we detail below, this measure has another natural advantage: the volume of the boundary controls the typical distance from a random input to the boundary of its linear region (see \S\ref{S:highD-inf}). This measures the stability of the function computed by the network, and it is intuitively related to robustness under adversarial perturbation.

\begin{figure}[htb]
\includegraphics[width=.23\textwidth]{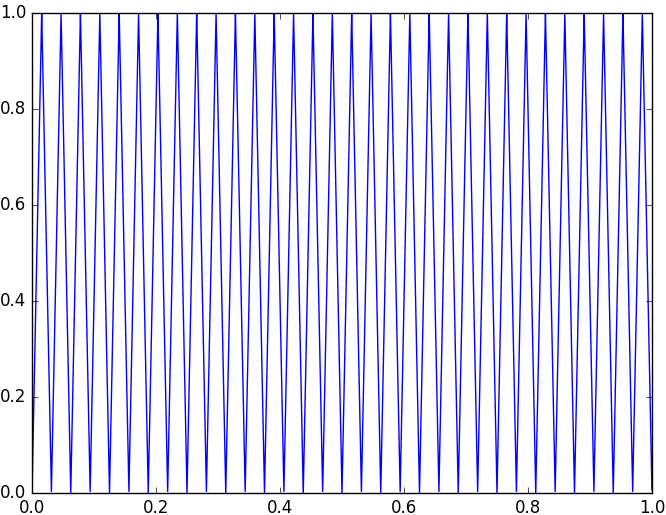}
\includegraphics[width=.23\textwidth]{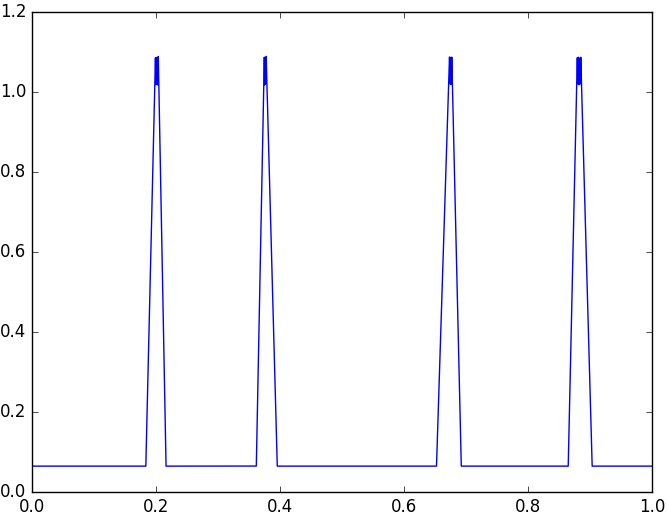}
\vskip.2in
\includegraphics[width=.46\textwidth]{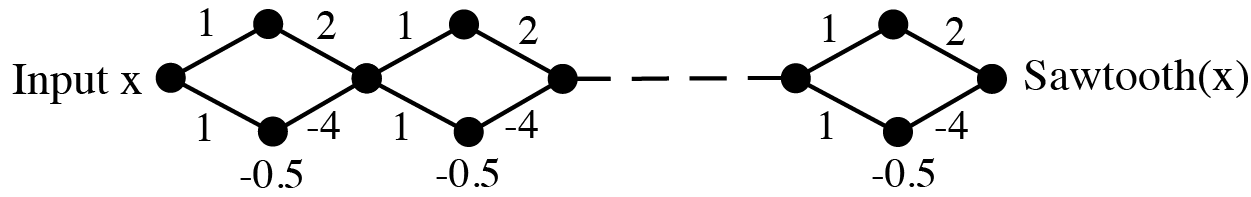}
\caption{The sawtooth function on the left with $2^n$ teeth can be expressed succinctly by a ReLU network with only $3n+4$ neurons (construction from \citet{telgarsky2015representation}). However, slight perturbation of the weights and biases of the network (by Gaussian noise with standard deviation $0.1$) greatly simplifies the linear regions captured by the network.}
\label{fig:sawtooth}
\end{figure}


\textbf{Our Contributions.} In this paper, we provide mathematical tools for analyzing the complexity of linear regions of a network with piecewise linear activations (such as ReLU) before, during, and after training. Our main contributions are as follows:
\begin{itemize}
    \item For networks at initialization, the total surface area of the boundary between linear regions scales as the number of neurons times the number of breakpoints of the activation function. This is our main result, from which several corollaries follow (see Theorem \ref{T:main}, Corollary \ref{C:init-formal-intro}, and the discussion in \S\ref{S:inf}).
    \item In particular, for any line segment through input space, the average number of regions intersecting it is \emph{linear} in the number of neurons, far below the exponential number of regions that is theoretically attainable.
    \item Theorem \ref{T:main} also allows us to conclude that, at initialization, the average distance from a sample point to the nearest region boundary is bounded below by a constant times the reciprocal of the number of neurons (see Corollary \ref{C:dist-formal-intro}).
    \item We find empirically that both the number of regions and the distance to the nearest region boundary stay roughly constant during training and in particular are far from their theoretical maxima. That this should be the case is strongly suggested by Theorem \ref{T:main}, though not a direct consequence of it.
\end{itemize}

Overall, our results stress that practical expressivity lags significantly behind theoretical expressivity. Moreover, both our theoretical and empirical findings suggest that for certain measures of complexity, trained deep networks are remarkably similar to the same networks at initialization. 

In the next section, we informally state our theoretical and empirical results and explore the underlying intuitions. Detailed descriptions of our experiments are provided in \S\ref{S:experiments}.  The precise theorem statements for ReLU networks can be found in \S\ref{S:formal}. The exact formulations for general piecewise linear networks are in Appendix \ref{S:formal-general}, with proofs  in the rest of the Supplementary Material. In particular, Appendix \ref{S:outlines} contains intuition for how our proofs are shaped, while details are completed in \S\ref{S:main-proof}-\ref{S:init-pf}.
\begin{figure}[htb]
  \centering
\includegraphics[scale=0.3]{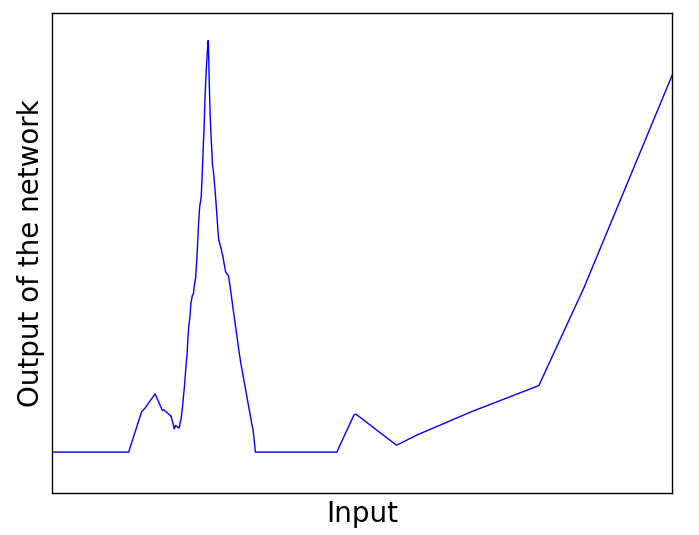}
\caption{Graph of function computed by a $\Relu$ net with input and output dimension $1$ at initialization. The weights of the network are He normal (i.i.d.~normal with variance = $2/$fan-in) and the biases are i.i.d.~normal with variance $10^{-6}$.}
\label{fig:1d_regions}
\end{figure}
\section{Informal Overview of Results}\label{S:inf}
This section gives an informal introduction to our results. We begin in \S \ref{S:1d-res} by describing the case of networks with input dimension $1.$ In \S \ref{S:highD-inf}, we consider networks with higher input dimension. For simplicity, we focus throughout this section on fully connected ReLU networks. We emphasize, however, that our results apply to any piecewise linear activation. Moreover, the upper bounds we present in Theorems \ref{T:1d-inf}, \ref{C:dist}, and \ref{T:main} (and hence in Corollaries \ref{C:init-formal-intro} and \ref{C:dist-formal-intro}) can also be generalized to hold for feed-forward networks with arbitrary connectivity, though we do not go into details in this work, for the sake of clarity of exposition.

\subsection{Number of Regions in 1D}\label{S:1d-res} Consider the simple case of a $\Relu$ net $\mN$ with input and output dimensions equal to $1.$ Such a network computes a piecewise linear function (see Figure \ref{fig:1d_regions}), and we are interested in understanding both at initialization and during training the number of distinct linear regions. There is a simple universal upper bound:
\begin{equation}\label{E:univ-UB}
\text{max }\# \set{\mathrm{regions}}~\leq~ 2^{\# \mathrm{neurons}},
\end{equation}
where the maximum is over all settings of weight and biases. This bound depends on the architecture of $\mN$ only via the number of neurons. For more refined upper bounds which take into account the widths of the layers, see Theorem 1 in \citet{raghu2017expressive} and Theorem 1 in \citet{serra2017bounding}. 

The constructions in \citet{montufar2014number, telgarsky2015representation, raghu2017expressive, serra2017bounding} indicate that the bound in \eqref{E:univ-UB} is very far from sharp for shallow and wide networks but that exponential growth in the number of regions can be achieved in deep, skinny networks for very special choices of weights and biases. This is a manifestation of the expressive power of depth, the idea that repeated compositions allow deep networks to capture complex hierarchical relations more efficiently per parameter than their shallow cousins. However, there is no non-trivial lower bound for the number of linear regions:
\[\text{min }\#\set{\mathrm{regions}}~=~1,\qquad \forall \mN.\]
The minimum is attained by setting all weights and biases to $0.$ This raises the question of the behavior for the average number of regions when the weights and biases are chosen at random (e.g. at initialization). Intuitively, configurations of weights and biases that come close to saturating the exponential upper bound \eqref{E:univ-UB} are numerically unstable in the sense that a small random perturbation of the weights and biases drastically reduces the number of linear regions (see Figure \ref{fig:sawtooth} for an illustration). In this direction, we prove a somewhat surprising answer to the question of how many regions $\mN$ has at initialization. We state the result for $\Relu$ but note that it holds for any piecewise linear, continuous activation function (see Theorems \ref{T:main} and \ref{T:main-general}).

\begin{theorem}[informal]{\label{T:1d-inf}}
Let $\mN$ be a network with piecewise linear activation with input and output dimensions of $\mN$ both equal $1$. Suppose the weights and biases are randomly initialized so that for each neuron $z$, its pre-activation $z(x)$ has bounded mean gradient 
\begin{equation}\label{E:good-init}
    \E{\norm{\nabla z(x)}}\leq C,\qquad \text{some }C>0.
\end{equation}
This holds, for example, for $\Relu$ networks initialized with independent, zero-centered weights with variance $2/\text{fan-in}.$ Then, for each subset $I\subset \R$ of inputs, the average number of linear regions inside $I$ is proportional to the number of neurons times the length of $I$
\[\E{\#\set{\mathrm{regions~in~}I}} ~\approx~  \abs{I}\cdot T \cdot \#\set{\mathrm{neurons}},\]
where $T$ is the number of breakpoints in the non-linearity of $\mN$ (for ReLU nets, $T=1$). The same result holds when computing the number of linear regions along any fixed $1$-dimensional curve in a high-dimensional input space. 
\end{theorem}

This theorem implies that the average number of regions along a one-dimensional curve in input space is proportional to the number of neurons, but independent of the arrangement of those neurons. In particular, a shallow network and a deep network will have the same complexity, by this measure, as long as they have the same total number of neurons. Of course, as $\abs{I}$ grows, the bounds in Theorem \ref{T:1d-inf} become less sharp. We plan to extend our results to obtain bounds on the total number of regions on all of $\R$ in the future. In particular, we believe that at initialization the mean total number of linear regions $\mN$ is proportional to the number of neurons (this is borne out in Figure \ref{fig:num_1d_regions}, which computes the total number of regions on an infinite line).

Theorem \ref{T:1d-inf} defies the common intuition that, on average, each layer in $\mN$ multiplies the number of regions formed up to the previous layer by a constant larger than one. This would imply that the average number of regions is exponential in the depth. To provide intuition for why this is not true for \textit{random} weights and biases, consider the effect of each neuron separately. Suppose the pre-activation $z(x)$ of a neuron $z$ satisfies $\abs{z'(x)}=\Theta(1)$, a hallmark of any reasonable initialization. Then, over a compact set of inputs, the piecewise linear function $x\mapsto z(x)$ cannot be highly oscillatory over a large portion of the range of $z$. Thus, if the bias $b_z$ is not too concentrated on any interval, we expect the equation $z(x)=b_z$ to have $O(1)$ solutions. On average, then, we expect that each neuron adds a \textit{constant number} of new linear regions. Thus, the average total number of regions should scale roughly as the number of neurons. 

Theorem \ref{T:1d-inf} follows from a general result, Theorem \ref{T:main}, that holds for essentially any non-degenerate distribution of weights and biases and with any input dimension. If $\norm{\nabla z(x)}$ and the bias distribution $\rho_{b_z}$ are well-behaved, then throughout training, Theorem \ref{T:main} suggests the number of linear regions along a $1$-dimensional curve in input space scales like the number of neurons in $\mN$. Figures \ref{fig:num_1d_regions}-\ref{fig:dist_to_boundary} show experiments that give empirical verification of this heuristic. 
\begin{figure*}[htb]
\centering
\includegraphics[width=.42\textwidth]{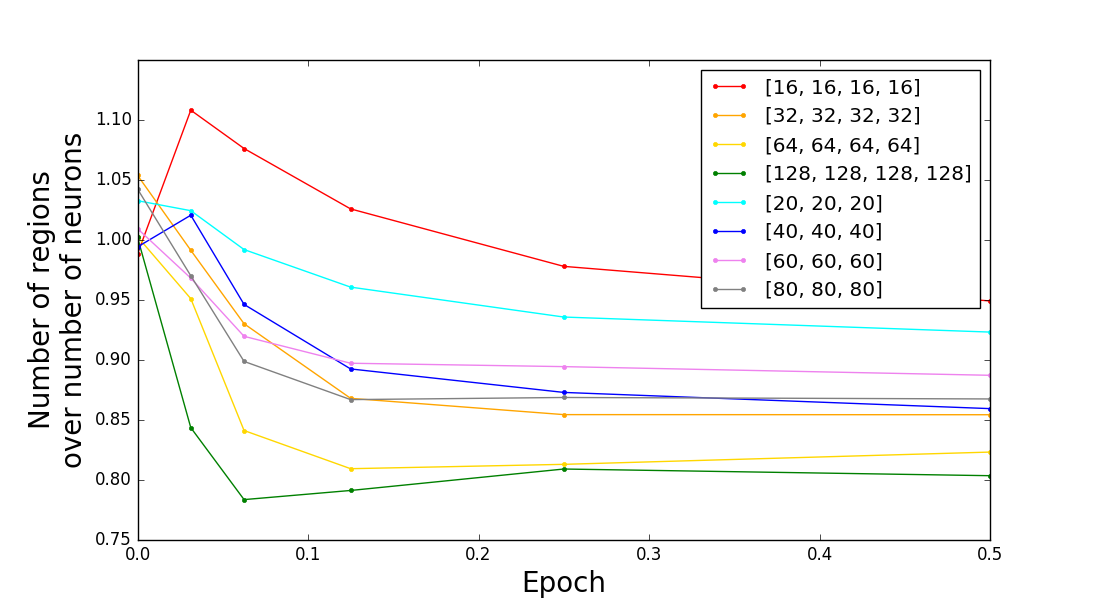}
\includegraphics[width=.42\textwidth]{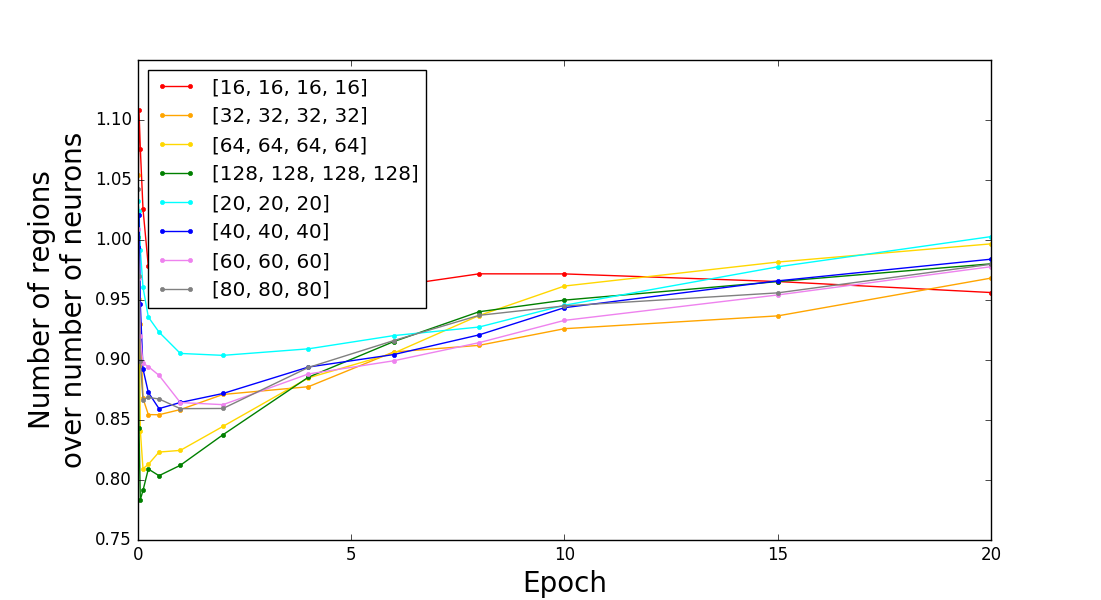}
\caption{We here show how the number of regions along 1D lines in input space changes during training. In accordance with Theorem \ref{T:main}, we scale the number of regions by the number of neurons. Plots show (a) early training, up through 0.5 epochs, and (b) later training, up through 20 epochs. Note that for all networks, number of regions is a fixed constant times the number of neurons at initialization, as predicted, and that the number decreases (slightly) early in training before rebounding. $[n_1,n_2,n_3]$ in the legend corresponds to an architecture with layer widths $784\text{ (input)},n_1,n_2,n_3,10\text{ (output)}$.}
\label{fig:num_1d_regions}
\end{figure*}
\subsection{Higher-Dimensional Regions}\label{S:highD-inf} For networks with input dimension exceeding $1,$ there are several ways to generalize counting linear regions. A unit-matching heuristic applied to Theorem \ref{T:1d-inf} suggests
\[\#\set{\mathrm{regions}}~=~\#\set{\mathrm{neurons}}^{\nin},\quad \nin=\text{input dim}.\]
Proving this statement is work in progress by the authors. Instead, we consider here a natural and, in our view, equally important generalization. Namely, for a bounded $K\subset \R^{\nin}$, we consider the $(\nin-1)$-dimensional volume density 
\begin{align}
\label{E:vol-den-def-2}\vol_{\nin-1}\lr{\mB_{\mN}\cap K}~\big/\vol_{\nin}(K),
\end{align}
where
\begin{equation}\label{E:B-def}
\mB_{\mN}=\set{x~|~\nabla \mN(x)\text{ is not continuous at }x}\end{equation}
is the boundary of the linear regions for $\mN$. When $\nin=1$, 
\[\vol_{0}\lr{\mB_{\mN}\cap K}~+~1~~=~~\#\set{\mathrm{regions~in~}K},\]
and hence the volume density \eqref{E:vol-den-def-2} truly generalizes to higher input dimension of the number of regions. One reason for studying the volume density \eqref{E:vol-den-def-2} is that it gives bounds from below for $\mathrm{distance}\lr{x,\mB_{\mN}}$, which in turn provides insight into the nature of the computation performed by $\mN.$ Indeed, the exact formula
\[\mathrm{distance}\lr{x,\mB_{\mN}}=\min_{\mathrm{neurons~z}}\left\{\abs{z(x)-b_z}\big/\norm{\nabla z(x)}\right\},\]
shows that $\mathrm{distance}\lr{x,\mB_{\mN}}$ measures the sensitivity over neurons at a given input $x$. In this formula,  $z(x)$ denotes the pre-activation for a neuron $z$ and $b_z$ is its bias, so that $\Relu(z(x)-b_z)$ is the post-activation. Moreover, the distance from a typical point to $\mB_{\mN}$ gives a heuristic lower bound for the typical distance to an adversarial example: two inputs closer than the typical distance to a linear region boundary likely fall into the same linear region, and hence are unlikely to be classified differently. Our next result generalizes Theorem \ref{T:1d-inf}. 

\begin{theorem}[informal]\label{C:dist}
Let $\mN$ be a network with a piecewise linear activation, input dimension $\nin$ and output dimension $1.$ Suppose its weights and biases are randomly initialized as in \eqref{E:good-init}. Then, for $K\subset \R^{d_{in}}$  bounded, the average volume of the linear region boundaries in $K$ satisfies:
\[\E{\frac{\vol_{\nin-1}\lr{\mB_{\mN}\cap K}}{\vol_{\nin}(K)}}~\approx~T\cdot \#\set{\mathrm{neurons}},\]
where $T$ is the number of breakpoints in the non-linearity of $\mN$ (for ReLU nets, $T=1$). Moreover, if $x\in [0,1]^{\nin}$ is uniformly distributed, then the average, over both $x$ and the weights/biases of $\mN$, distance from $x$ to $\mB_{\mN}$ satisfies
\[\E{\mathrm{distance}\lr{x,\mB_{\mN}}}~\geq~ C \lr{\# \set{\mathrm{neurons}}}^{-1},\quad C>0.\]
\end{theorem}

Experimentally, $\mathrm{distance}\lr{x,\mB_{\mN}}$ remains comparable to $\lr{\# \set{\mathrm{neurons}}}^{-1}$ throughout training (see Figure \ref{fig:dist_to_boundary}).

\section{Experiments}\label{S:experiments}
We empirically verified our theorems and further examined how linear regions of a network change during training. All experiments below were performed with fully-connected networks, initialized with He normal weights (i.i.d.~with variance $2/\text{fan-in}$) and biases drawn i.i.d.~normal with variance $10^{-6}$ (to prevent collapse of regions at initialization, which occurs when all biases are uniquely zero). Training was performed on the vectorized MNIST (input dimension 784) using the Adam optimizer at learning rate $10^{-3}$. All networks attain test accuracy in the range $95-98\%$.
\begin{figure*}[htb]
  \centering
\includegraphics[width=0.33\textwidth]{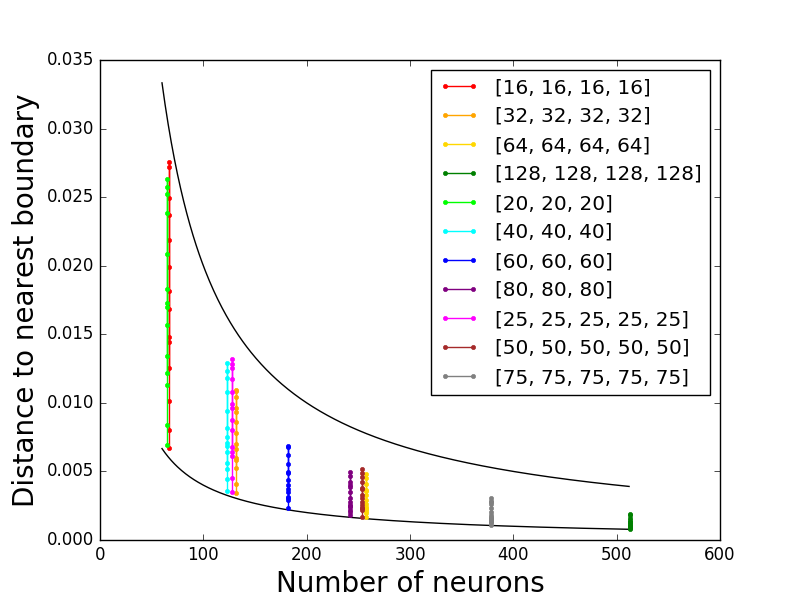}
\includegraphics[width=0.33\textwidth]{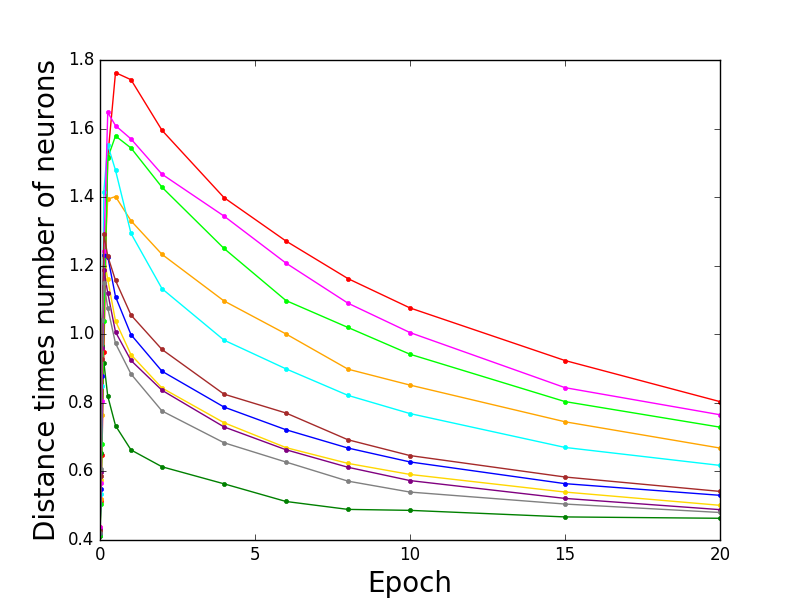}
\includegraphics[width=0.33\textwidth]{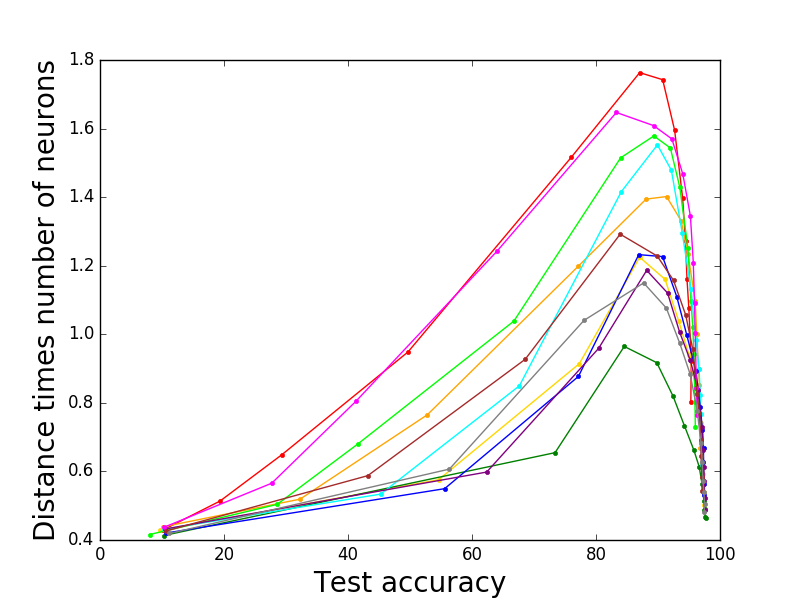}
\caption{We here consider the average distance to the nearest boundary, as evaluated over 10000 randomly selected sample points. In (a) we show that this distance is essentially bounded between $0.4/\#\set{\mathrm{neurons}}$ and $1.5/\#\set{\mathrm{neurons}}$. Accordingly, in the next plot, we normalize the distance to the nearest boundary by dividing by the number of neurons. We plot this quantity against (b) epoch and (c) test accuracy. Observe that, in keeping with the findings of Figure \ref{fig:num_1d_regions}, the distance to the nearest boundary first increases quickly (as the number of regions decreases), then rebounds more slowly as the network completes training. $[n_1,n_2,n_3]$ in the legend corresponds to an architecture with layer widths $784\text{ (input)},n_1,n_2,n_3,10\text{ (output)}$.}
\label{fig:dist_to_boundary}
\end{figure*}
\subsection{Number of Regions Along a Line}
We calculated the number of regions along lines through the origin and and a random selected training example in input space. For each setting of weights and biases within the network during training, the number of regions along each line is calculated exactly by building up the network one layer at a time and calculating how each region is split by the next layer of neurons. Figure \ref{fig:num_1d_regions} represents the average over 5 independent runs, from each of which we sample 100 lines; variance across the different runs is not significant.

Figure \ref{fig:num_1d_regions} plots the average number of regions along a line, divided by the number of neurons in the network, as a function of epoch during training. We make several observations:
\begin{enumerate}
    \item As predicted by Theorem \ref{T:main}, all networks start out with the number of regions along a line equal to a constant times the number of neurons in the network (the constant in fact appears very close to 1 in this case).
    \item Throughout training, the number of regions does not deviate significantly from the number of neurons in the network, staying within a small constant of the value at initialization, in keeping with our intuitive understanding of Theorem \ref{T:main} described informally around Theorem \ref{T:1d-inf} above.
    \item The number of regions actually decreases during the initial part of training, then increases again. We explore this behavior further in other experiments below.
\end{enumerate}
\subsection{Distance to the Nearest Region Boundary}
We calculated the average distance to the nearest boundary for $10000$ randomly selected input points, for various networks throughout training. Points were selected randomly from a normal distribution with mean and variance matching the componentwise mean and variance of MNIST training data. Results were averaged over $12$ independent runs, but variance across runs is not significant. Rerunning these experiments with sample points selected randomly from (i) the training data or (ii) the test data  yielded similar results to random sample points.

In keeping with our results in the preceding experiment, the distance to the nearest boundary first increases then decreases during training. As predicted by Theorem \ref{C:dist}, we find that for all networks, the distance to the nearest boundary is well-predicted by $1/\#\set{\mathrm{neurons}}$. Throughout training, we find that it approximately varies between the curves $0.4/\#\set{\mathrm{neurons}}$ and $1.5/\#\set{\mathrm{neurons}}$ (Figure \ref{fig:dist_to_boundary}(a)). At initialization, as we predict, all networks have the same value for the product of number of neurons and distance to the nearest region boundary (Figure \ref{fig:dist_to_boundary}(b)); these products then diverge (slightly) for different architectures, first increasing rapidly and then decreasing more slowly.

We find Figure \ref{fig:dist_to_boundary}(c) fascinating, though we do not completely understand it. It plots the product of number of neurons and distance to the nearest region boundary against the test accuracy. It suggests two phases of training: first regions expand, then they contract. This lines up with observations made in \citet{arpit2017closer} that neural networks ``learn patterns first'' on which generalization is simple and then refine the fit to encompass memorization of individual samples. A generalization phase would suggest that regions are growing, while memorization would suggest smaller regions are fit to individual data points. This is, however, speculation and more experimental (and theoretical) exploration will be required to confirm or disprove this intuition.
\begin{figure}[htb]
  \centering
\includegraphics[width=0.4\textwidth]{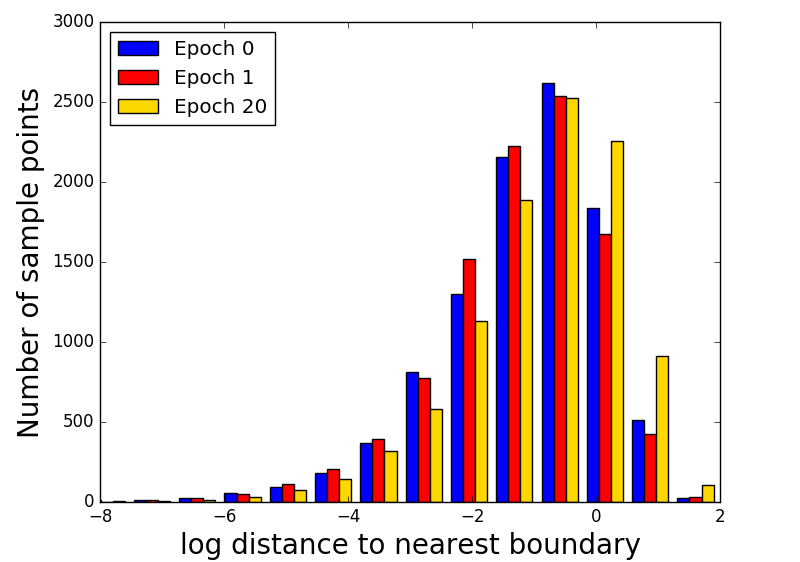}
\caption{Distribution of $\log$ distances from random sample points to the nearest region boundary for a network of depth 4 and width 16, at initialization and after 1 and 20 epochs of training on MNIST.}
\label{fig:distribution}
\vskip-.1in
\end{figure}
We found it instructive to consider the full distribution of distances from sample points to their nearest boundaries, rather than just the average. For a single network (depth 4, width 16), Figure \ref{fig:distribution} indicates that this distribution does not significantly change during training, although there appears to be a slight skew towards larger regions, in agreement with the findings in \citet{novak2018sensitivity}. The histogram shows $\log$-distances. Hence, distance to the nearest region boundary varies over many orders of magnitude. This is consistent with Figures \ref{fig:2d_regions_init} and \ref{fig:1d_regions}, which lend credence to the intuition that small distances to the nearest region boundary are explained by the presence of many small regions. According to Theorem \ref{T:main}, this should correlate with a combination of regions in input space at which some neurons have a large gradient and neurons with highly peaked biases distributions. We hope to return to this in future work.

\subsection{Regions Within a 2D Plane}
We visualized the regions of a network through training. Specifically, following experiments in \citet{novak2018sensitivity}, we plotted regions within a plane in the $784$-dimensional input space (Figure \ref{fig:2d_regions}) through three data points with different labels ($0$, $1$, and $2$, in our case) inside a square centered at the circumcenter of the three examples. The network shown has depth $3$ and width $64$. We observe that, as expected from our other plots, the regions expand initially during training and then contract again. We expect the number of regions within a $2$-dimensional subspace to be on the order of the square of the number of neurons -- that is, $(64 \cdot 3)^2 \approx 4\times 10^4$, which we indeed find.

Our approach for calculating regions is simple. We start with a single region (in this case, the square), and subdivide it by adding neurons to the network one by one. For each new neuron, we calculate the linear function it defines on each region, and determine whether that region is split into two. This approach terminates within a reasonable amount of time precisely because our theorem holds: there are relatively few regions. Note that we \emph{exactly} determine all regions within the given square by calculating all region boundaries; thus our counts are exact and do not miss any small regions, as might occur if we merely estimated regions by sampling points from input space. 
\begin{figure*}[htb]
    \centering
    \includegraphics[width=0.9\textwidth]{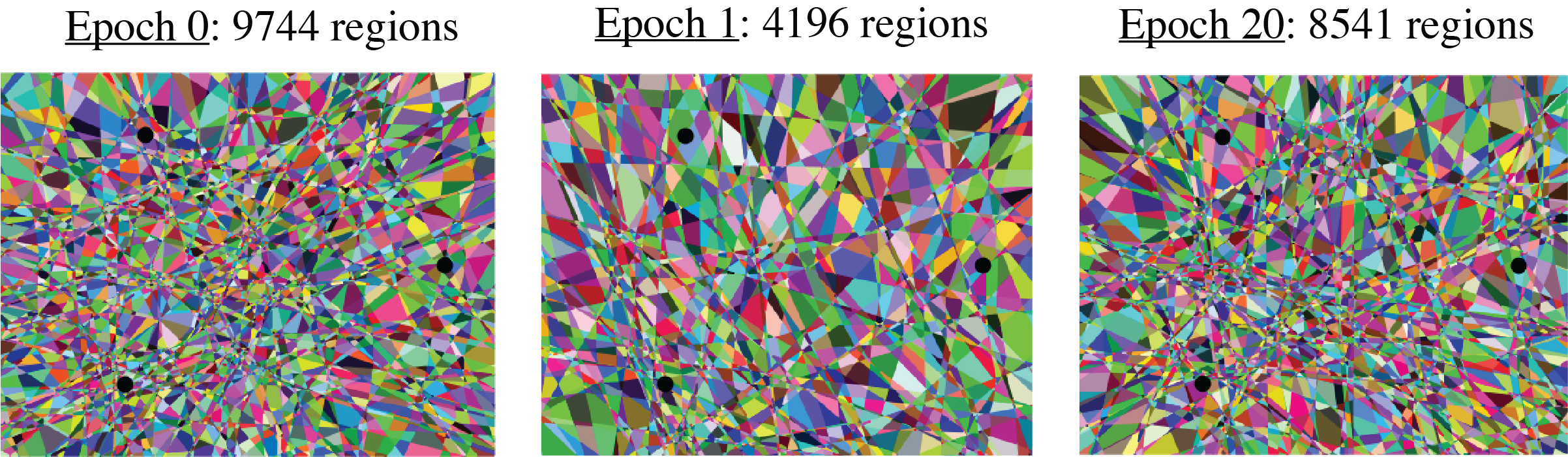}
    \caption{Here we show the linear regions that intersect a 2D plane through input space for a network of depth 3 and width 64 trained on MNIST. Black dots indicate the positions of the three MNIST training examples defining the plane. Note that we obtain qualitatively different pictures from \citet{novak2018sensitivity}, which may result partially from our using ReLU activation instead of ReLU6.}
    \label{fig:2d_regions}
\end{figure*}

\section{Related Work}\label{s:related-work} There are a number of works that touch on the themes of this article: (i) the expressivity of depth; (ii) counting the number of regions in networks with piecewise linear activations; (iii) the behavior of linear regions through training; and (iv) the difference between expressivity and learnability. Related to (i), we refer the reader to \citet{eldan2016power,telgarsky2016benefits} for examples of functions that can be efficiently represented by deep but not shallow ReLU nets. Next, still related to (i), for uniform approximation over classes of functions, again using deep ReLU nets, see \citet{yarotsky2017error, rolnick2017power, yarotsky2018optimal, petersen2018optimal}. For interesting results on (ii) about counting the maximal possible number of linear regions in networks with piecewise linear activations see \citet{bianchini2014complexity, montufar2014number, poole2016exponential, arora2016understanding, raghu2017expressive}. Next, in the vein of (iii), for both a theoretical and empirical perspective on the number of regions computed by deep networks and specifically how the regions change during training, see \citet{poole2016exponential, novak2018sensitivity}. In the direction of (iv), we refer the reader to \citet{shalev2017failures, hanin2018start, hanin2018neural}. Finally, for general insights into learnability and expressivity in deep vs.~shallow networks see \citet{ mhaskar2016deep, mhaskar2016learning,  zhang2016understanding, lin2017does, poggio2017and, neyshabur2017exploring}.

\section{Formal Statement of Results}\label{S:formal}
To state our results precisely, we fix some notation. Let $d,\nin, n_1,\ldots, n_{d}\geq 1$ and consider a depth $d$ fully connected $\Relu$ net $\mathcal N$ with input dimension $\nin$, output dimension $1$, and hidden layer widths $n_j,\,j=1,\ldots, d-1.$ As explained in the introduction, a generic configuration of its weights and biases partitions the input space $\R^{\nin}$ into a union of polytopes $P_j$ with disjoint interiors. Restricted to each $P_j,$ $\mathcal N$ computes a linear function. 

Our main mathematical result, Theorem \ref{T:main}, concerns the set $\mB_{\mN}$ of points $x\in \R^{\nin}$ at which the gradient $\nabla \mN$ is discontinuous at $x$ (see \eqref{E:B-def}). For each $k=1,\ldots, \nin,$ we define 
\begin{equation}\label{E:Bk-def}
\mB_{\mN,k}=\text{the }``(\nin-k)\text{--dimensional piece'' of }\mB_{\mN}.
\end{equation}
More precisely, we set $\mB_{\mN,0}:=\emptyset$ and recursively define $\mB_{\mN,k}$ to be the set of points $x\in \mB_{\mN}\backslash \set{\mB_{\mN,0}\cup\cdots\cup\mB_{\mN,k-1}}$ so that
in a neighborhood of $x$ the set $\mB_{\mN}\backslash \set{\mB_{\mN,0}\cup\cdots\cup\mB_{\mN,k-1}}$ coincides with a co-dimension $k$ hyperplane.

For example, when $\nin=2,$ the linear regions $P_j$ are polygons, the set $\mB_{\mN,1}$ is the union of the open line segments making up the boundaries of the $P_j$, and $\mB_{\mN,2}$ is the collection of vertices of the $P_j.$ Theorem \ref{T:main} provides a convenient formula for the average of the $(\nin-k)-$dimensional volume of $\mB_{\mN,k}$ inside any bounded, measurable set $K\subset \R^{n_\mathrm{in}}$. To state the result, for every neuron $z$ in $\mathcal N$ we will write
\begin{equation}\label{E:z-notation}
z(x)~:=~\text{pre-activation at }z,\quad \ell(z)~=~\text{layer index of }z 
\end{equation}
and $b_z~:=~\text{bias at }z.$ Thus, for a given input $x\in \R^{n_0}$, the post-activation of $z$ is 
\begin{equation}\label{E:Z-notation}
Z(x):=\Relu(z(x))=\max\set{0,z(x)-b_z}.
\end{equation}
Theorem \ref{T:main} holds under the following assumption on the distribution of weights and biases: 
\begin{enumerate}
\item[\textbf{A1:}] The conditional distribution of any collection of biases $b_{z_1},\ldots, b_{z_k}$, given all the other weights and biases, has a density $\rho_{b_{z_1},\ldots, b_{z_k}}(b_1,\ldots, b_k)$ with respect to Lebesgue measure on $\R^{k}$.
\item[\textbf{A2:}] The joint distribution of all the weights has a density with respect to Lebesgue measure on $\R^{\#\text{weights}}$.
\end{enumerate}
These assumptions hold in particular when the weights and biases of $\mN$ are independent with marginal distributions that have a density relative to Lebesgue measure on $\R$ (i.e. at initialization). They hold much more generally, however, and can intuitively be viewed as a non-degeneracy assumption on the behavior of the weights and biases of $\mN$. Specifically, they are used in Proposition \ref{P:k-dim-pts} to ensure that the set $\mB_{\mN, k}$ consists of inputs where exactly $k$ neurons turn off/on. Assumption (A1) also allows us, in Proposition \ref{P:vol-rep}, to apply the co-area formula \eqref{E:co-area} to compute the expect volume of the set of inputs where a given collection of neurons turn on/off. Our main result is the following.

\begin{theorem}\label{T:main}
Suppose $\mathcal N$ is a feed-forward $\Relu$ net with input dimension $n_0,$ output dimension $1$, and random weights/biases. Assume that the distribution of weights/biases satisfies Assumptions $A1$ and $A2$ above. Then, with the notation \eqref{E:z-notation}, for any bounded measurable set $K\subseteq \R^{n_{\mathrm{in}}}$ and any $k=1,\ldots, n_{\mathrm{in}},$ the average $(n_{\mathrm{in}}-k)-$ dimensional volume $\E{\vol_{n_{\mathrm{in}}-k}(\mB_{\mN,k}\cap K)}$ of $\mB_{\mN,k}$ inside $K$ is
 \begin{equation}\label{E:vol-est-gen}
 \E{\vol_{n_{\mathrm{in}}-k}(\mB_{\mN,k}\cap K)}
 \end{equation}
of $\mB_{\mN,k}$ inside $K$ is, in the notation \eqref{E:z-notation}, 
\[\sum_{\substack{\mathrm{distinct~neurons~} \\z_1,\ldots, z_k\mathrm{~in~}\mN}}~\int_K \mathbb E\big[Y_{z_1,\ldots, z_k}(x)\big]dx,\]
where $Y_{z_1,\ldots, z_k}(x)$ is
\[\norm{J_{z_1,\ldots, z_k}(x)} \,\rho_{b_{z_1},\ldots, b_{z_k}}(z_1(x),\ldots, z_k(x))\]
times the indicator function of the event that $z_j\text{ is good at }x$ for each $j=1,\ldots, k$. Here, $J_{z_1,\ldots, z_k}$ is the $k\x \nin$ Jacobian of the map $x\mapsto (z_1(x),\ldots,z_k(x)),$ 
\[\norm{J_{z_1,\ldots, z_k}(x)}:=\det\lr{J_{z_1,\ldots, z_k}(x)\lr{J_{z_1,\ldots, z_k}(x)}^T}^{1/2},\]
the function $\rho_{b_{z_1},\ldots, b_{z_k}}$ is the density of the joint distribution of the biases $b_{z_1},\ldots, b_{z_k}$, and we say a neuron $z$ is good at $x$ if there exists a path of neurons from $z$ to the output in the computational graph of $\mN$ so that each neuron along this path is open at $x$).
\end{theorem}
 To evaluate the expression in \eqref{E:vol-est-gen} requires information on the distribution of gradients $\nabla z(x)$, the pre-activations $z(x)$, and the biases $b_z.$ Exact information about these quantities is available at initialization \citep{hanin2018neural,hanin2018start,hanin2018products}, yielding the following Corollary.
 
\begin{corollary}\label{C:init-formal-intro}
   With the notation and assumptions of Theorem \ref{T:main}, suppose the weights are independent are drawn from a fixed probability measure $\mu$ on $\R$ that is symmetric around $0$ and then rescaled to have $\Var[\mathrm{weights}] \,=\, 2/\text{fan-in}$. Fix $k\in \set{1,\ldots,n_{\mathrm{in}}}$. Then there exists $C>0$ for which
 \begin{align}\label{E:UB-init-intro}
&\frac{\E{\vol_{n_{\mathrm{in}}-k}(\mB_{\mN,k}\cap K)}}{\vol_{n_{\mathrm{in}}}(K)}\\
\notag&\quad\leq~ \binom{\#\set{\mathrm{neurons}}}{k} (C_{\mathrm{grad}}\cdot C_{\mathrm{bias}})^k,
\end{align}
where
\[C_{\mathrm{bias}}=\sup_z\sup_{b\in \R} \rho_{b_z}(b)\]
and
\[C_{\mathrm{grad}}=\sup_{z}\sup_{x\in\R^{\nin}} \E{\norm{\nabla z(x)}^{2k}}^{1/k}~\leq~Ce^{C\sum_{j=1}^d\frac{1}{n_j}}\]
where $C>0$ depends only on $\mu$ but not on the architecture of $\mN$ and $n_j$ is the width of the $j^{th}$ hidden layer. Moreover, we also have similar lower bounds
 \begin{equation}\label{E:LB-init}
\binom{\#\set{\mathrm{neurons}}}{k} c_{\mathrm{bias}}^k~\leq~ \frac{\E{\vol_{n_{\mathrm{in}}-k}(\mB_{\mN,k}\cap K)}}{\vol_{n_{\mathrm{in}}}(K)}
\end{equation}
where
\[c_{\mathrm{bias}}~=~\inf_{\abs{b}\leq\eta} \rho_{b_z}(b),\]
and
\[\eta =\lr{\frac{\sup_{x\in K}\norm{x}^2}{n_{\mathrm{in}}}+\sum_{j=1}^d \sigma_{b_j}^2} e^{C'\sum_{j=1}^d\frac{1}{n_j}},\]
with $C'>0$ depending only on the distribution $\mu$ of the weights in $\mN$.
\end{corollary}


We prove Corollary \ref{C:init-formal} in Appendix \ref{S:init-pf}. Let us state one final corollary of Theorem \ref{T:main}
\begin{corollary}\label{C:dist-formal-intro}
Suppose $\mN$ is as in Theorem \ref{T:main} and satisfies the hypothesis \eqref{E:bounded-grad} in Corollary \ref{C:init-formal}. Then, for any compact set $K\subset \R^{n_{\mathrm{in}}}$ let $x$ be a uniform point in $K.$ There exists $c>0$ independent of $K$ so that
\[\E{\mathrm{distance(x,\mB_{\mN})}}\geq \frac{c~ }{C_{\mathrm{bias}}C_{\mathrm{grad}}\#\set{\mathrm{neurons}}}.\]
\end{corollary}
We prove Corollary \ref{C:dist-formal} in \S \ref{S:dist-pf}. The basic idea is simple. For every $\epsilon>0,$ we have
\[\E{\mathrm{distance(x,\mB_{\mN})}}\geq \epsilon\P\lr{\mathrm{distance}(x,\mB_{\mN})>\epsilon},\]
with the probability on the right hand side scaling like \[1-\vol_{n_{\mathrm{in}}}(T_\epsilon(\mB_{\mN})\cap K)/\vol_{n_{\mathrm{in}}}(K),\]
where $T_\epsilon(\mB_{\mN})$ is the tube of radius $\epsilon$ around $\mB_{\mN}.$ We expect that its volume like $\epsilon \vol_{n_{\mathrm{in}}-1}(\mB_{\mN})$. Taking $\ep=c/\#\set{\mathrm{neurons}}$ yields the conclusion of Corollary \ref{C:dist-formal}.

\section{Conclusions and Further Work}

The question of why depth is powerful has been a persistent problem for deep learning theory, and one that recently has been answered by works giving enhanced expressivity as the ultimate explanation. However, our results suggest that such explanations may be misleading. While we do not speak to all notions of expressivity in this paper, we have both theoretically and empirically evaluated one common measure: the linear regions in the partition of input space defined by a network with piecewise linear activations. We found that the average size of the boundary of these linear regions depends only on the number of neurons and not on the network depth -- both at initialization and during training. This strongly suggests that deeper networks do not learn more complex functions than shallow networks. We plan to test this interpretation further in future work -- for example, with experiments on more complex tasks, as well as by investigating higher order statistics, such as the variance.

We do not propose a replacement theory for the success of deep learning; however, prior work has already hinted at how such a theory might proceed. Notably, \citet{ba2014deep} show that, once deep networks are trained to perform a task successfully, their behavior can often be replicated by shallow networks, suggesting that the advantages of depth may be linked to easier learning.



\bibliographystyle{icml2019}
\bibliography{references.bib}

\newpage
\appendix

\section{Formal Statement of Results for General Piecewise Linear Activations}\label{S:formal-general}
In \S\ref{S:formal}, we stated our results in the case of ReLU activation, and now frame these results for a general piecewise linear non-linearity. We fix some notation. Let $\phi:\R\gives \R$ be a continuous piecewise linear function with $T$ breakpoints $\xi_0=-\infty< \xi_1<\xi_2<\cdots< \xi_T<\xi_{T+1}=\infty.$ That is, there exist $p_j,q_j\in \R$ so that
\begin{equation}\label{E:phi-def}
t\in [\xi_j,\xi_{j+1}]\quad \Rightarrow\quad \phi(t)=q_j t + p_j,\; q_{j}\neq q_{j+1}.
\end{equation}
The analog of Theorem \ref{T:main} for general $\phi$ is the following.
\begin{theorem}\label{T:main-general}
Let $\phi:\R\gives \R$ be a continuous piecewise linear function with $T$ breakpoints $\xi_1<\cdots<\xi_T$ as in \eqref{E:phi-def}. Suppose $\mN$ is a fully connected network with input dimension $\nin,$ output dimension $1$, random weights and biases satisfying $A1$ and $A2$ above, and non-linearity $\phi$.

Let $J_{z_1,\ldots, z_k}$ be the $k\x \nin$ Jacobian of the map $x\mapsto (z_1(x),\ldots,z_k(x)),$ 
\[\norm{J_{z_1,\ldots, z_k}(x)}:=\det\lr{J_{z_1,\ldots, z_k}(x)\lr{J_{z_1,\ldots, z_k}(x)}^T}^{1/2},\]
and write $\rho_{b_{z_1},\ldots, b_{z_k}}$ for the density of the joint distribution of the biases $b_{z_1},\ldots, b_{z_k}$. We say a neuron $z$ is \emph{good} at $x$ if there exists a path of neurons from $z$ to the output in the computational graph of $\mN$ so that each neuron $\widehat{z}$ along this path is open at $x$ (i.e.~$\phi'(\widehat{z}(x)-b_{\widehat{z}})\neq 0$).

Then, for any bounded, measurable set $K\subseteq \R^{\nin}$ and any $k=1,\ldots, \nin,$ the average $(\nin-k)$--dimensional volume 
\[\E{\vol_{\nin-k}(\mB_{\mN,k}\cap K)}\] 
of $\mB_{\mN,k}$ inside $K$ is, in the notation of \eqref{E:z-notation}, 
 \begin{equation}\label{E:vol-est}
\sum_{\substack{\mathrm{distinct~neurons~} \\z_1,\ldots, z_k\mathrm{~in~}\mN}}\sum_{i_1,\ldots, i_k=1}^T~\int_K \mathbb E\big[Y_{z_1,\ldots, z_k}^{(\xi_{i_1},\ldots, \xi_{i_k})}(x)\big]dx,
\end{equation}
where $Y_{z_1,\ldots, z_k}^{(\xi_{i_1},\ldots, \xi_{i_k})}(x)$ equals
\begin{equation}\label{E:Y-def}
\norm{J_{z_1,\ldots, z_k}(x)} \,\rho_{b_{z_1},\ldots, b_{z_k}}(z_1(x)-\xi_{i_1},\ldots, z_k(x)-\xi_{i_k})
\end{equation}
multiplied by the indicator function of the event that $z_j$ is good at $x$ for every $j.$
\end{theorem}

Note that if in the definition \eqref{E:phi-def} of $\phi$ we have that the possible values $\phi'(t)\in \set{q_0,\ldots, q_T}$ do not include $0$, then we may ignore the event that $z_j$ are good at $x$ in the definition of $Y_{z_1,\ldots, z_k}^{(\xi_{i_1},\ldots, \xi_{i_k})}.$

\begin{corollary}\label{C:init-formal}
   With the notation and assumptions of Theorem \ref{T:main-general}, suppose in addition that the weights and biases are independent. Fix $k\in \set{1,\ldots,\nin}$ and suppose that for every collection of distinct neurons $z_1,\ldots, z_k$, the average magnitude of the product of gradients is uniformly bounded:
 \begin{equation}\label{E:bounded-grad}
\sup_{\substack{\mathrm{neurons~}z_1,\ldots, z_k\\\mathrm{inputs}~x}}\E{\prod_{j=1}^k\norm{\nabla z_j(x)}}\leq C_{\mathrm{grad}}^k.
\end{equation}
Then we have the following upper bounds
 \begin{align}\label{E:UB-init}
&\frac{\E{\vol_{\nin-k}(\mB_{\mN,k}\cap K)}}{\vol_{\nin}(K)}\\
\notag&\quad\leq~ \binom{\#\set{\mathrm{neurons}}}{k} (T\cdot 2 C_{\mathrm{grad}}C_{\mathrm{bias}})^k,
\end{align}
where $T$ is the number of breakpoints in the non-linearity $\phi$ of $\mN$ (see \eqref{E:phi-def}) and
\[C_{\mathrm{bias}}=\sup_z\sup_{b\in \R} \rho_{b_z}(b).\]
\end{corollary}


We prove Corollary \ref{C:init-formal} in \S\ref{S:init-pf} and state a final corollary of Theorem \ref{T:main}:
\begin{corollary}\label{C:dist-formal}
Suppose $\mN$ is as in Theorem \ref{T:main} and satisfies the hypothesis \eqref{E:bounded-grad} in Corollary \ref{C:init-formal} with constants $C_{\mathrm{bias}},\,C_{\mathrm{grad}}$. Then, for any compact set $K\subset \R^{\nin}$ let $x$ be a uniform point in $K.$ There exists $c>0$ independent of $K$ so that
\[\E{\mathrm{distance(x,\mB_{\mN})}}\geq \frac{c~ T }{C_{\mathrm{bias}}C_{\mathrm{grad}}\#\set{\mathrm{neurons}}},\]
where, as before, $T$ is the number of breakpoints in the non-linearity $\phi$ of $\mN$. 
\end{corollary}
We prove Corollary \ref{C:dist-formal} in \S\ref{S:dist-pf}. The basic idea is simple. For every $\epsilon>0,$ we have
\[\E{\mathrm{distance(x,\mB_{\mN})}}\geq \epsilon\P\lr{\mathrm{distance}(x,\mB_{\mN})>\epsilon},\]
with the probability on the right hand side scaling like \[1-\vol_{\nin}(T_\epsilon(\mB_{\mN})\cap K)\big/\vol_{\nin}(K),\]
where $T_\epsilon(\mB_{\mN})$ is the tube of radius $\epsilon$ around $\mB_{\mN}.$ We expect that its volume like $\epsilon \vol_{\nin-1}(\mB_{\mN})$. Taking $\ep=c/\#\set{\mathrm{neurons}}$ yields the conclusion of Corollary \ref{C:dist-formal}.

\section{Outline of Proof of Theorem \ref{T:main-general}}\label{S:outlines} The purpose of this section is to give an intuitive explanation of the proof of Theorem \ref{T:main}. We fix a non-linearity $\phi:\R\gives \R$ with breakpoints $\xi_1<\cdots<\xi_T$ (as in \eqref{E:phi-def}) and consider a fully connected network $\mN$ with input dimension $\nin\geq 1$, output dimension $1$, and non-linearity $\phi.$ For each neuron $z$ in $\mN$, we write
\begin{equation}\label{E:ell-def}
  \ell(z)~:=~\text{layer index of }z
\end{equation}
and set
\begin{equation}\label{E:S-def}
S_z:=\set{x\in \R^{\nin}~|~z(x)-b_z\in \set{\xi_1,\ldots, \xi_T}}.
\end{equation}
We further 
\begin{equation}\label{E:S-twiddle-def}
  \twiddle{S}_z:=S_z\cap \mO,
\end{equation}
where 
\[\mO:=\left\{x\in \R^{\nin}~\bigg|~\substack{\forall ~j=1,\ldots, d~~\exists\text{ neuron }z\text{ with }\\\ell(z)=j\text{ s.t. }\phi'(z(x)-b_z)\neq 0}\right\}.\]
Intuitively, the set $S_z$ is the collection of inputs for which the neuron $z$ turns from on to off. In contrast, the set $\mO$ is the collection of inputs $x\in \R^{\nin}$ for which $\mN$ is open in the sense that there is a path from the input to the output of $\mN$ so that all neurons along this path compute are not constant in a neighborhood $x$. Thus, $\twiddle{S}_z$ is the set of inputs at which neuron $z$ switches between its linear regions and at which the output of neuron $z$ actually affects the function computed by $\mN.$

We remark here that $\mO=\emptyset$ if in the non-linearity $\phi$ there are no linear pieces at which the slopes on $\phi$ equals $0$ (i.e.~$q_j\neq 0$ for all $j$ in the definition \eqref{E:phi-def} of $\phi$). If, for example, $\phi$ is ReLU, then $\mO$ need not be empty. 

The overall proof of Theorem \ref{T:main} can be divided into several steps. The first gives the following representation of $\mB_{\mN}.$
\begin{proposition}\label{P:B-decomp}
Under Assumptions $A1$ and $A2$ of Theorem \ref{T:main}, we have, with probability $1,$
\[\mB_{\mN}~=~\bigcup_{\mathrm{neurons~}z}\twiddle{S}_z.\]
\end{proposition}
The precise proof of Proposition \ref{P:B-decomp} can be found in \S\ref{S:B-decomp-pf} below. The basic idea is that if for all $y$ near a fixed input $x\in \R^{\nin},$ none of the pre-activations $z(y)-b_z$ cross the boundary of a linear region for $\phi$, then $x\not \in \mB_{\mN}.$ Thus, $\mB_{\mN}\subset \bigcup_z S_z.$ Moreover, if a neuron $z$ satisfies $z(x)-b_z=S_i$ for some $i$ but there are no open paths from $z$ to the output of $\mN$ for inputs near $x$, then $z$ is dead at $x$ and hence does not influence $\mN$ at $x.$ Thus, we expect the more refined inclusion $\mB_{\mN}\subset \bigcup_z \twiddle{S}_z$. Finally, if $x\in \twiddle{S}_z$ for some $z$ then $x \in \mB_{\mN}$ unless the contribution from other neurons to $\nabla\mN(y)$ for $y$ near $x$ exactly cancels the discontinuity in $\nabla z(x).$ This happens with probability $0$. \\

The next step in proving Theorem \ref{T:main} is to identify the portions of $\mB_{\mN}$ of each dimension. To do this, we write for any distinct neurons $z_1,\ldots, z_k$,
\[\twiddle{S}_{z_1,\ldots, z_k}:=\bigcap_{j=1}^k \twiddle{S}_{z_j}.\]
The set $\twiddle{S}_{z_1,\ldots, z_k}$ is, intuitively, the collection of inputs at which $z_j(x)-b_{z_j}$ switches between linear regions for $\phi$ and at which the output of $\mN$ is affected by the post-activations of these neurons. Proposition \ref{P:B-decomp} shows that we may represent $\mB_{\mN}$ as a disjoint union
\[\mB_{\mN}=\bigcup_{k=1}^{\nin} \mB_{\mN,k},\]
where
\[ \mB_{\mN,k}:=\bigcup_{\substack{\text{distinct neurons}\\z_1,\ldots, z_k}} \twiddle{S}_{z_1,\ldots, z_k} \cap \lr{\bigcup_{z\neq z_1,\ldots, z_k} \twiddle{S}_z}^c.\]
In words, $\mB_{\mN,k}$ is the collection of inputs in $\mO$ at which exactly $k$ neurons turn from on to off. The following Proposition shows that $\mB_{\mN,k}$ is precisely the ``$(\nin-k)$-dimensional piece of $\mB_{\mN}$'' (see \eqref{E:Bk-def}). 
\begin{proposition}\label{P:k-dim-pts}
  Fix $k=1,\ldots, \nin,$ and $k$ distinct neurons $z_1,\ldots, z_k$ in $\mN.$ Then, with probability $1,$ for every $x\in \mB_{\mN,k}$ there exists a neighborhood in which $\mB_{\mN,k}$ coincides with a $(\nin-k)-$dimensional hyperplane. 
\end{proposition}

\noindent We prove Proposition \ref{P:k-dim-pts} in \S\ref{S:k-dim-pts-pf}. The idea is that each $\twiddle{S}_{z_1,\ldots, z_k}$ is piecewise linear and, with probability $1$, at every point at which \textit{exactly} the neurons $z_1,\ldots, z_k$ contribute to $\mB_{\mN}$, its co-dimension is the number of linear conditions needed to define it. Observe that with probability $1$, the bias vector $(b_{z_1},\ldots, b_{z_{k+1}})$ for any collection $z_1,\ldots, z_{k+1}$ of distinct neurons is a regular value for $x\mapsto (z_1(x),\ldots, z_{k+1}(x))$. Hence, 
\[\vol_{\nin-k}\lr{\twiddle{S}_{z_1,\ldots, z_{k+1}}}=0.\]
Proposition \ref{P:k-dim-pts} thus implies that, with probability $1,$
\[\vol_{\nin-k}\lr{\mB_{\mN,k}}=\sum_{\substack{\mathrm{distinct~neurons~} \\z_1,\ldots, z_k}} \vol_{\nin-k}\lr{\twiddle{S}_{z_1,\ldots, z_k}}.\]
The final step in the proof of Theorem \ref{T:main} is therefore to prove the following result.
\begin{proposition}\label{P:vol-rep}
  Let $z_1,\ldots, z_k$ be distinct neurons in $\mN.$ Then, for any bounded, measurable $K\subset \R^{\nin}$,
  \begin{align*}
&\E{\vol_{\nin-k}\lr{\twiddle{S}_{z_1,\ldots, z_k}}}\\
&\quad=~\int_K \sum_{i_1,\ldots, i_k=1}^T\E{Y_{z_1,\ldots, z_k}^{(S_{i_1},\ldots, S_{i_k})}(x)}dx,
  \end{align*}
where $Y_{z_1,\ldots, z_k}^{(S_{i_1},\ldots, S_{i_k})}$ is defined as in \eqref{E:Y-def}. 
\end{proposition}
\noindent We provide a detailed proof of Proposition \ref{P:vol-rep} in \S\ref{S:vol-rep-pf}. The intuition is that the image of the volume element $dx$ under $x\mapsto z(x)-S_i$ is the volume element
\[\norm{J_{z_1,\ldots, z_k}(x)}\, dx\]
from \eqref{E:Y-def}. The probability of an infinitesimal neighborhood $dx$ of $x$ belonging to a $(\nin-k)$-dimensional piece of $\mB_{\mN}$ is therefore the probability 
\begin{align*}
&\rho_{b_{z_1},\ldots, b_{z_k}}(z_1(x)-S_{i_1},\ldots, z_k(x)-S_{i_k})\\
&\quad\times~ \norm{J_{z_1,\ldots, z_k}(x)}\, dx
\end{align*}
that the vector of biases $(b_{z_j},\,\,j=1,\ldots,k)$ belongs to the image of $dx$ under map $\lr{z_j(x)-S_{i_j},j=1,\ldots,k}$ for some collection of breakpoints $S_{i_j}.$ The formal argument uses the co-area formula (see \eqref{E:co-area} and \eqref{E:co-area-2}).

\section{Proof of Theorem \ref{T:main}}
\label{S:main-proof}

\subsection{Proof of Proposition \ref{P:B-decomp}}\label{S:B-decomp-pf}
Recall that the non-linearity $\phi:\R\gives \R$ is continuous and piecewise linear with $T$ breakpoints $\xi_1<\cdots<\xi_T,$ so that, with $\xi_0=-\infty,\, \xi_{T+1}=\infty$, we have
\[t\in (\xi_i,\xi_{i+1})\quad \Rightarrow\quad \phi(t)=q_it + p_i\]
with $q_i\neq q_{i+1}.$ For each $x\in \R^{\nin},$ write
\begin{align*}
  Z_x^+&:=\left\{z~\big|~ z(x)-b_z\in (\xi_i,\xi_{i+1})\text{ and }q_i\neq 0\text{  for some }i\right\}\\
  Z_x^-&:=\left\{z~\big|~ z(x)-b_z\in (\xi_i,\xi_{i+1})\text{ and }q_i=0\text{  for some }i\right\}\\
  Z_x^0&:=\left\{z~\big|~ z(x)-b_z=\xi_i\text{  for some }i\right\} 
\end{align*}
Intuitively, $Z_x^+$ are the neurons that, at the input $x$ are open (i.e.~contribute to the gradient of the output $\mN(x)$) but do not change their contribution in a neighborhood of $x$, $Z_x^-$ are the neurons that are closed, and $Z_x^0$ are the neurons that, at $x$, produce a discontinuity in the derivative of $\mN.$ Thus, for example, if $\phi=\Relu,$ then
\[Z_x^*:=\set{z~|~\mathrm{sgn}(z(x)-b_z)=*},\quad *\in \set{+,-,0}.\]
We begin by proving that $\mB_{\mN}\subseteq \bigcup_z \twiddle{S}_z$ by checking the contrapositive
\begin{equation}\label{E:inclusion-1}
\lr{\bigcup_z \twiddle{S}_z}^c~\subseteq \mB_{\mN}^c.
\end{equation}
Fix $x\in \lr{\bigcup_z \twiddle{S}_z}^c$. Note that $Z_x^{\pm}$ are locally constant in the sense that there exists $\ep>0$ so that for all $y$ with $\norm{y-x}<\ep$, we have
\begin{equation}\label{E:locally-const}
Z_x^-\subseteq Z_y^-,\quad Z_x^+\subseteq Z_y^+,\quad Z_y^+\cup Z_y^0\subseteq Z_x^+\cup Z_x^0.
\end{equation}
Moreover, observe that if in the definition \eqref{E:phi-def} of $\phi$ none of the slopes $q_i$ equal $0$, then $Z_y^-=\emptyset$ for every $y$. To prove \eqref{E:inclusion-1}, consider any path $\gamma$ from the input to the output in the computational graph of $\mN.$ Such a path consists of $d+1$ neurons, one in each layer:
\[\gamma=\lr{z_\gamma^{(0)},\ldots, z_\gamma^{(d)}},\,\, \ell(z_\gamma^{(j)})=j.\]
To each path we may associate a sequence of weights:
\[w_\gamma^{(j)}~:=~\text{weight connecting }z_\gamma^{(j-1)}\text{ to }z_\gamma^{(j)},\quad j=1,\ldots,d.\]
We will also define
\[q_\gamma^{(j)}(x)~:=~\sum_{i=0}^T q_i {\bf 1}_{\set{z_\gamma^{(x)}-b_{z_\gamma^{(j)}}\in (\xi_i,\xi_{i+1}]}}.\]
For instance, if $\phi=\Relu$, then 
\[q_\gamma^{(j)}(x)~=~{\bf 1}_{\set{z_\gamma^{(j)}(x)-b_z\geq 0}},\]
and in general only one term in the definition of $q_\gamma^{(j)}(x)$ is non-zero for each $z.$ We may write 
\begin{equation}\label{E:N-rep}
\mN(y)=\sum_{i=1}^{\nin}y_i\sum_{\substack{\text{paths }\gamma:i\gives\text{out}}} \prod_{j=1}^d q_\gamma^{(j)}(y)w_\gamma^{(j)}~+~\mathrm{constant},
\end{equation}
Note that if $x\in \lr{\bigcup_z \twiddle{S}_z}^c$, then for any path $\gamma$ through a neuron $z\in Z_x^0$, we have
\[\exists~~j\text{ s.t. } z_\gamma^{(j)}\in Z_x^-.\]
This is an open condition in light of \eqref{E:locally-const}, and hence for all $y$ in a neighborhood of $x$ and for any path $\gamma$ through a neuron $z\in Z_x^0$ we also have that 
\[\exists~~j\text{ s.t. } z_\gamma^{(j)}\in Z_y^-.\]
Thus, since the summand in \eqref{E:N-rep} vanishes identically if $\gamma\cap Z_y^-\neq \emptyset$, we find that for $y$ in a neighborhood of any $x\in \lr{\bigcup_z \twiddle{S}_z}^c$ we may write
\begin{equation}\label{E:N-rep-2}
\mN(y)=\sum_{i=1}^{\nin}y_i\sum_{\substack{\text{paths }\gamma:i\gives\text{out}\\ \gamma\subset Z_x^+}} \prod_{j=1}^d q_\gamma^{(j)}(y)w_\gamma^{(j)}~+~\mathrm{constant}.
\end{equation}
But, again by \eqref{E:locally-const}, for any fixed $x$, all $y$ in a neighborhood of $x$ and each $z\in Z_x^+,$ we have $z\in Z_y^+$ as well. Thus, in particular, 
\[z(x)-b_z\in (\xi_{i},\xi_{i+1})\quad \Rightarrow\quad z(y)-b_z\in (\xi_i,\xi_{i+1}).\]
Thus, for $y$ sufficiently close to $x,$ we have for every path in the sum \eqref{E:N-rep-2} that 
\[q_\gamma^{(j)}(y)=q_\gamma^{(j)}(x).\]
Therefore, the partial derivatives $(\partial \mN / \partial y_i)(y)$ are independent of $y$ in a neighborhood of $x$ and hence continuous at $x$. This proves \eqref{E:inclusion-1}. Let us now prove the reverse inclusion:
\begin{equation}\label{E:inclusion-2}
\bigcup_{z}\twiddle{S}_z~~\subseteq~~\mB_{\mN}
\end{equation}
Note that, with probability $1,$ we have
\[\vol_{n_{\mathrm{in}-1}}(S_{z_1}\cap S_{z_2})=0\]
for any pair of distinct neurons $z_1,z_2.$ Note also that since $x\mapsto \mN(x)$ is continuous and piecewise linear, the set $\mB_{\mN}$ is closed. Thus, it is enough to show the slightly weaker inclusion 
\begin{equation}\label{E:inclusion-3}
\bigcup_{z}\lr{\twiddle{S}_z\big\backslash \bigcup_{\widehat{z}\neq z}S_{\widehat{z}}}~~\subseteq~~\mB_{\mN}
\end{equation}
since the closure of $\twiddle{S}_z\big\backslash \bigcup_{\widehat{z}\neq z}S_{\widehat{z}}$ equals $\twiddle{S}_z.$
Fix a neuron $z$ and suppose $x\in \twiddle{S}_z\big\backslash \bigcup_{\widehat{z}\neq z}S_{\widehat{z}}$. By definition, we have that for every neuron $\widehat{z}\neq z,$ either 
\[\widehat{z}\in Z_x^+\quad \text{or}\quad \widehat{z}\in Z_x^-.\]
This has two consequences. First, by \eqref{E:locally-const}, the map $y\mapsto z(y)$ is linear in a neighborhood of $x.$ Second, in a neighborhood of $x,$ the set $\twiddle{S}_z$ coincides with $S_z$. Hence, combining these facts, near $x$ the set $\twiddle{S}_z$ coincides with the hyperplane 
\begin{equation}\label{E:hyperplane}
\set{x~|~z(x)-b_z=\xi_i},\qquad \text{for some }i.
\end{equation}
We may take two sequences of inputs $y_n^+,y_n^-$ on opposite sides of this hyperplane so that
\[\lim_{n\gives \infty }y_n^+~=~\lim_{n\gives \infty }y_n^-~=~x\]
and
\[\phi'(z(y_n^{+})-b_z)=q_{i},\quad \phi'(z(y_n^{+})-b_z)=q_{i-1},~~~\forall n,\]
where the index $i$ the same as the one that defines the hyperplane \eqref{E:hyperplane}. Further, since $\mB_{\mN}$ has co-dimension $1$ (it is contained in the piecewise linear co-dimension $1$ set $\bigcup_z S_z$, for example), we may also assume that $y_n^+,y_n^-\not \in \mB_{\mN}.$ Consider any path $\gamma$ from the input to the output of the computational graph of $\mN$ passing through $z$ (so that  $z=z_\gamma^{(j)}\in \gamma$). By construction, for every $n$, we have
\[q_\gamma^{(j)}(y_n^+)\neq q_{\gamma}^{(j)}(y_n^-),\]
and hence, after passing to a subsequence, we may assume that the symmetric difference
\begin{equation}\label{E:path-set}
Z_{y_n^+}^+\Delta Z_{y_n^-}^+~\neq~ \emptyset
\end{equation}
of the paths that contribute to the representation \eqref{E:N-rep} for $y_{n}^+,\,y_n^-$ is fixed and non-empty (the latter since it always contains $z$). For any $y\not \in \mB_{\mN},$ we may write, for each $i=1,\ldots, \nin$
\begin{equation}\label{E:partial-diff}
\frac{\partial \mN}{\partial y_i}(y) ~=~\sum_{\substack{\mathrm{paths~}\gamma:i\gives\mathrm{out}\\ \gamma \subset Z_{y}^+}}\prod_{j=1}^d q_\gamma^{(j)}(y)w_\gamma^{(j)}.
\end{equation}
Substituting into this expression $y=y_n^{\pm}$, we find that there exists a non-empty collection $\Gamma$ of paths from the input to the output of $\mN$ so that
\[\frac{\partial \mN}{\partial y_i}(y_n^+)-\frac{\partial \mN}{\partial y_i}(y_n^-)~=~\sum_{\gamma \in \Gamma}a_j\prod_{j=1}^d c_\gamma^{(j)} w_\gamma^{(j)}\]
where
\[a_j\in \set{-1,1},\qquad c_\gamma^{(j)}\in \set{q_0,\ldots, q_T}.\]
Note that the expression above is a polynomial in the weights of $\mN$. Note also that, by construction, this polynomial is not identically zero due to the condition \eqref{E:path-set}. There are only finitely many such polynomials since both $a_j$ and $c_\gamma^{(j)}$ range over a finite alphabet. For each such non-zero polynomial, the set of weights at which it vanishes has co-dimension $1$. Hence, with probability $1,$ the difference $\frac{\partial \mN}{\partial y_i}(y_n^+)-\frac{\partial \mN}{\partial y_i}(y_n^-)$ is non-zero. This shows that the partial derivatives $\frac{\partial \mN}{\partial y_i}$ are not continuous at $x$ and hence that $x\in \mB_{\mN}.$ \hfill $\square$

\subsection{Proof of Proposition \ref{P:k-dim-pts}}\label{S:k-dim-pts-pf}
Fix distinct neurons $z_1,\ldots, z_k$ and suppose $x\in \twiddle{S}_{z_1,\ldots, z_k}$ but not in $\twiddle{S}_z$ for any $z\neq z_1,\ldots, z_k.$ After relabeling, we may assume that they are ordered by layer index:
\[\ell(z_1)~\leq~\cdots~\leq~\ell(z_k).\]
Since $x\in \mO$, we also have that $x\not \in S_z$ for any $z\neq z_1,\ldots, z_k.$ Thus, there exists a neighborhood $U$ of $x$ so $S_z \cap U=\emptyset$ for every $z\neq z_1,\ldots, z_k.$ Thus, there exists a neighborhood of $x$ on which $y\mapsto z_1(y)$ is linear. 

Hence, as explained near \eqref{E:hyperplane} above, $\twiddle{S}_{z_1}$ is a hyperplane near $x.$ We now restrict our inputs to this hyperplane and repeat this reasoning to see that, near $x,$ the set $\twiddle{S}_{z_1,z_2}$ is a hyperplane inside $\twiddle{S}_{z_1}$ and hence, near $x$, is the intersection of two hyperplanes in $\R^{\nin}$. Continuing in this way shows that in a neighborhood of $x,$ the set $\twiddle{S}_{z_1,\ldots, z_k}$ is equal to the intersection of $k$ hyperplanes in $\R^{\nin}.$ Thus, $\twiddle{S}_{z_1,\ldots, z_k}\backslash \lr{\bigcup_{z\neq z_1,\ldots, z_k}\twiddle{S}_z}^c$ is precisely the intersection of $k$ hyperplanes in a neighborhood of each of its points.
\hfill $\square$

\subsection{Proof of Proposition \ref{P:vol-rep}}\label{S:vol-rep-pf}
Let $z_1,\ldots, z_k$ be distinct neurons in $\mN,$ and fix a compact set $K\subset \R^{\nin}$. We seek to compute the mean of $\vol_{\nin-k}\lr{\twiddle{S}_{z_1,\ldots, z_k}\cap K}$, which we may rewrite as
\begin{align}\label{E:integral-1}
&\int_{S_{z_1,\ldots, z_k}\cap K}{\bf 1}_{\left\{\substack{z_j\text{ is good at }x\\j=1,\ldots, k}\right\}}\,\mathrm{dvol}_{\nin-k}(x)\\
\quad &=\sum_{i_1,\ldots,i_k=1}^T \int_{S_{z_1,\ldots, z_k}^{(\xi_{i_1},\ldots, \xi_{i_k})}\cap K}{\bf 1}_{\left\{\substack{z_j\text{ is good at }x\\j=1,\ldots, k}\right\}}\mathrm{dvol}_{\nin-k}(x),\notag
\end{align}
where we've set 
\[S_{z_1,\ldots, z_k}^{(\xi_{i_1},\ldots, \xi_{i_k})}=\set{x~|~z_{j}(x)-b_{z_j}=\xi_{i_{j}},\,\,\,j=1,\ldots,k}.\]
Note that the map $x\mapsto \lr{z_1(x),\ldots, z_k(x)}$ is Lipschitz, and recall the co-area formula, which says that if $\psi\in L^1(\R^n)$ and $g:\R^n\gives \R^m$ with $m\leq n$ is Lipschitz, then 
\begin{equation}\label{E:co-area}
\int_{\R^m}\int_{g^{-1}(t)} \psi(x)\,\mathrm{dvol}_{n-m}(x)dt
\end{equation}
equals
\begin{equation}\label{E:co-area-2}
\int_{\R^n} \psi(x)\norm{Jg(x)}\,\mathrm{dvol}_n(x),
\end{equation}
where $Jg$ is the $m\x n$ Jacobian of $g$ and 
\[\norm{Jg(x)}=\det\lr{(Jg(x))(Jg(x))^T}^{1/2}.\]
We assumed that the biases $b_{z_1},\ldots, b_{z_j}$ have a joint conditional density 
\[\rho_{{\bf b}_{{\bf z}}}=\rho_{b_{z_1},\ldots, b_{z_k}}\]
given all other weights and biases. The mean of the term in \eqref{E:integral-1} corresponding to a fixed $\xi=\lr{\xi_{i_1},\ldots, \xi_{i_k}}$ over the conditional distribution of $b_{z_1},\ldots, b_{z_j}$ is therefore
\[\int_{\R^k}d{\bf b}\rho_{{\bf b}_{{\bf z}}}({\bf b}) \int_{\set{{\bf z}-{\bf b}={\bf \xi}}\cap K}{\bf 1}_{\left\{\substack{z_j\text{ is good at }x\\j=1,\ldots, k}\right\}}\,\mathrm{dvol}_{\nin-k}(x),\]
where we've abbreviated ${\bf b}=\lr{b_1,\ldots, b_k}$ as well as ${\bf z}(x)=\lr{z_1(x),\ldots, z_k(x)}$. This can rewritten as
\[\int_{\R^k}d{\bf b} \int_{\set{{\bf z}={\bf b}}\cap K}\rho_{{\bf b}_{{\bf z}}} ({\bf z}(x)-{\bf \xi}){\bf 1}_{\left\{\substack{z_j\text{ is good at }x\\j=1,\ldots, k}\right\}}\mathrm{dvol}_{n_0-k}(x).\]
Thus, applying the co-area formula \eqref{E:co-area}, \eqref{E:co-area-2} shows that the average of \eqref{E:integral-1} over the conditional distribution of $b_{z_1},\ldots, b_{z_j}$ is precisely
\[\int_K Y_{z_1,\ldots, z_k}(x)\, dx.\]
Taking the average over the remaining weighs and biases, we may commute the expectation $\E{\cdot}$ with the $dx$ integral since the integrand is non-negative. This completes the proof of Proposition \ref{P:vol-rep}.\hfill $\square$

\section{Proof of Corollary \ref{C:init-formal}}\label{S:init-pf}
\noindent We begin by proving the upper bound in \eqref{E:UB-init}. By Theorem \ref{T:main}, $\E{\vol\lr{\mB_{\mN,k}\cap K}}$ equals
\[\sum_{\substack{\mathrm{distinct~neurons~}}\\z_1,\ldots, z_k}\sum_{i_1,\ldots, i_k=1}^T\int_K \E{Y_{z_1,\ldots, z_k}^{(\xi_{i_1},\ldots, \xi_{i_k})}(x)}(x)dx,\]
where, as in \eqref{E:Y-def}, $Y_{z_1,\ldots, z_k}^{(\xi_{i_1},\ldots, \xi_{i_k})}(x)$ is
\[\norm{J_{z_1,\ldots, z_k}(x)} \,\rho_{b_{z_1},\ldots, b_{z_k}}(z_1(x)-\xi_{i_1},\ldots, z_k(z)-\xi_{i_k})\]
times the indicator function of the even that $z_j$ is good at $x$ for every $j.$ When the weights and biases of $\mN$ are independent, we may write $\rho_{b_{z_1},\ldots, b_{z_k}}(b_1,\ldots, b_k)$ as
\[\prod_{j=1}^k \rho_{b_{z_j}}(b_j)~\leq~ \lr{\sup_{\mathrm{neurons~}z}\sup_{b\in \R}\rho_{b_z}(b)}^k=C_{\mathrm{bias}}^k.\]
Hence, 
\[Y_{z_1,\ldots, z_k}(x)~\leq~ C_{\mathrm{bias}}^k\, \lr{\det \lr{J_{z_1,\ldots, z_k}(x)\lr{J_{z_1,\ldots, z_k}(x)}^T}}^{1/2}.\]
Note that 
\[J_{z_1,\ldots, z_k}(x)\lr{J_{z_1,\ldots, z_k}(x)}^T=\mathrm{Gram}\lr{\nabla z_1(x),\ldots, \nabla z_k(x)},\]
where for any $v_i\in \R^n$
\[\mathrm{Gram}(v_1,\ldots, v_k)_{i,j}=\inprod{v_i}{v_j}\]
is the associated Gram matrix. The Gram identity says that $\det \lr{J_{z_1,\ldots, z_k}(x)\lr{J_{z_1,\ldots, z_k}(x)}^T}^{1/2}$ equals
\[\norm{\nabla z_1(x)\wedge\cdots \wedge \nabla z_k(x)},\]
which is the the $k$-dimensional volume of the parallelopiped in $\R^{\nin}$ spanned by $\set{\nabla z_j(x),\, j=1,\ldots, k}.$ We thus have
\[\det \lr{J_{z_1,\ldots, z_k}(x)\lr{J_{z_1,\ldots, z_k}(x)}^T}^{1/2}~\leq~\prod_{j=1}^k \norm{\nabla z_j(x)}.\]
The estimate \eqref{E:bounded-grad} proves the upper bound \eqref{E:UB-init}. For the special case of $\phi=\Relu$ we use the AM-GM inequality and Jensen's inequality to write
\begin{align*}
\E{\prod_{j=1}^k \norm{\nabla z_j(x)}}&\leq \E{\lr{\frac{1}{k}\sum_{j=1}^k \norm{\nabla z_j(x)}}^k}\\
&\leq\frac{1}{k}\sum_{j=1}^k \E{\norm{\nabla z_j}^k}.
\end{align*}
Therefore, by Theorem 1 of \citet{hanin2018products}, there exist $C_1,C_2>0$ so that 
\[\E{\prod_{j=1}^k \norm{\nabla z_j(x)}}\leq \lr{C_1e^{C_2 \sum_{j=1}^d \frac{1}{n_j}}}^k.\]
This completes the proof of the upper bound in \eqref{E:UB-init}. To prove the power bound, lower bound in \eqref{E:UB-init} we must argue in a different way. Namely, we will induct on $k$ and use the following facts to prove the base case $k=1$:
\begin{enumerate}
\item At initialization, for each fixed input $x,$ the random variables $\set{{\bf 1}_{\set{z(x)>b_z}}}$ are independent Bernoulli random variables with parameter $1/2.$ This fact is proved in Proposition 2 of \citet{hanin2018products}. In particular, the event $\set{z\text{ is good at }x}$, which occurs when there exists a layer $j\in \ell(z)+1,\ldots, d$ in which $z(x)\leq b_z$ for every neuron, is independent of $\set{z(x),\,b_z}$ and satisfies
  \begin{equation}\label{E:open-est}
\P\lr{z\text{ is good at }x}\geq 1 -  \sum_{j=1}^d 2^{-n_j}.
\end{equation}
\item At initialization, for each fixed input $x$, we have
  \begin{equation}\label{E:squared-act-est}
\frac{1}{2}\E{z(x)^2}=\frac{\norm{x}^2}{\nin}+\sum_{j=1}^{\ell(z)}\sigma_{b_j}^2,
\end{equation}
where $\sigma_{b_j}^2:=\Var[\mathrm{biases~at~layer~}j]$. This is Equation (11) in the proof of Theorem 5 from \citet{hanin2018start}.
\item At initialization, for every neuron $z$ and each input $x,$ we have
\begin{equation}\label{E:grad-est}
\E{\norm{\nabla z(x)}^2}=2.
\end{equation}
This follows easily from Theorem 1 of \citet{hanin2018neural}.
\item At initialization, for each $1\leq j \leq \nin$ and every $x\in \R^{\nin}$
  \begin{equation}\label{E:log-grad-est}
\E{\log\lr{\nin \lr{\frac{\partial z}{\partial x_j}(x)}^2}}~~=~~ -\frac{5}{2}\sum_{j=1}^{\ell(z)}\frac{1}{n_j}
\end{equation}
plus $O\lr{ \sum_{j=1}^{\ell(z)}\frac{1}{n_j^2}}$, where $n_j$ is the width of the $j^{\text{th}}$ hidden layer and the implied constant depends only on the $4^{\text{th}}$ moment of the measure $\mu$ according to which weights are distributed. This estimate follows immediately by combining Corollary 26 and Proposition 28 in \citet{hanin2018products}.
\end{enumerate}
We begin by proving the lower bound in \eqref{E:UB-init} when $k=1.$ We use \eqref{E:open-est} to see that $\E{\vol_{\nin-1}\lr{\mB_{\mN}\cap K}}$ is bounded below by 
\[\bigg(1-\sum_{j=1}^d2^{-n_j}\bigg) \sum_{\mathrm{neurons~}z}\int_K \E{\norm{\nabla z(x)}\rho_{b_z}(z(x))}\, dx.\]
Next, we bound the integrand. Fix $x\in \R^{\nin}$ and a parameter $\eta>0$ to be chosen later. The integrand $\E{\norm{\nabla z(x)}\rho_{b_z}(z(x))}$ is bounded below by
\begin{align*}
  & \E{\norm{\nabla z(x)}\rho_{b_z}(z(x)){\bf 1}_{\set{\abs{z(x)}}\leq \eta}}\\
&~~~\geq~ \left[\inf_{\abs{b}\leq \eta}\rho_{b_z}(b) \right]\E{\norm{\nabla z(x)}{\bf 1}_{\set{\abs{z(x)}\leq \eta}}},
\end{align*}
which is bounded below by
\[\left[\inf_{\abs{b}\leq \eta}\rho_{b_z}(b) \right]\left[\E{\norm{\nabla z(x)}} -\E{\norm{\nabla z(x)}{\bf 1}_{\set{\abs{z(x)}}> \eta}} \right].\]
Using Cauchy-Schwarz, the term $\E{\norm{\nabla z(x)}{\bf 1}_{\set{\abs{z(x)}}> \eta}}$ is bounded above by 
\[\lr{\E{\norm{\nabla z(x)}}^2\P\lr{\abs{z(x)}> \eta}}^{1/2},\]
which using \eqref{E:grad-est} and \eqref{E:squared-act-est} together with Markov's inequality, is bounded above by 
\[\frac{2}{\eta^{1/2}}\lr{\frac{\norm{x}^2}{ \nin} +\sum_{j=1}^{\ell(z)}\sigma_{b_j}^2}^{1/2}.\]
Next, using Jensen's inequality twice, we write
\begin{align*}
\log \E{\norm{\nabla z(x)}}&~\geq~ \frac{1}{2}\E{\log\lr{\norm{\nabla z(x)}^2}}\\
&~=~\frac{1}{2}\E{\log\lr{\sum_{j=1}^{\nin} \lr{\frac{\partial z}{\partial x_j}(x)}^2}}\\
&~\geq~\frac{1}{2}\E{\log \lr{\nin^{1/2} \frac{\partial z}{\partial x_j}(x)}^2}\\
&~=~-\frac{5}{4}\sum_{j=1}^{\ell(z)}\frac{1}{n_j}+O\lr{\sum_{j=1}^{\ell(z)}\frac{1}{n_j^2}},
\end{align*}
where in the last inequality we applied \eqref{E:log-grad-est}. Putting this all together, we find that exists $c>0$ so that
\[  \E{\norm{\nabla z(x)}\rho_{b_z}(z(x))}~\geq~ c \left[\inf_{\abs{b}\leq \eta}\rho_{b_z}(b) \right],\]
where 
\[\eta~\geq~4\lr{\frac{\norm{x}^2}{\nin}+\sum_{j=1}^d \sigma_{b_j}^2}\, e^{\frac{5}{4}\sum_{j=1}^d \frac{1}{n_j} + O\lr{\sum_{j=1}^{\ell(z)}\frac{1}{n_j^2}}}.\]
In particular, we may take 
\[\eta=\lr{\frac{\sup_{x\in K}\norm{x}^2}{\nin}+\sum_{j=1}^d \sigma_{b_j}^2} e^{C\sum_{j=1}^d\frac{1}{n_j}}\]
for $C$ sufficiently large. This completes the proof of the lower bound in \eqref{E:UB-init} when $k=1$. To complete the proof of Corollary \ref{C:init-formal}, suppose we have proved the lower bound in \eqref{E:UB-init} for all $\Relu$ networks $\mN$ and all collections of $k-1$ distinct neurons. We may assume after relabeling  that the neurons $z_1,\ldots, z_k$ are ordered by layer index:
\[\ell(z_1)\leq \cdots\leq \ell(z_k).\]
With probability $1,$ the set $S_{z_1}\subset \R^{\nin}$ is piecewise linear, co-dimension $1$ with finitely many pieces, which we denote by $P_\alpha$. We may therefore rewrite $\vol_{\nin-k}\lr{\twiddle{S}_{z_1,\ldots, z_k}\cap K}$ as
\[\sum_\alpha \vol_{\nin-k}\lr{\twiddle{S}_{z_2,\ldots, z_k}\cap P_\alpha \cap K}.\]
We now define a new neural network $\mN_\alpha$, obtained by restricting $\mN$ to $P_\alpha.$ The input dimension for $\mN_\alpha$ equals $\nin-1,$ and the weights and biases of $\mN_\alpha$ satisfy all the assumptions of Corollary \ref{C:init-formal}. We can now apply our inductive hypothesis to the $k-1$ neurons $z_2,\ldots, z_k$ in $\mN_\alpha$ and to the set $K\cap P_\alpha.$ This gives
\begin{align*}
&\E{\sum_\alpha \vol_{\nin-k}\lr{\twiddle{S}_{z_2,\ldots, z_k}\cap P_\alpha \cap K}}\\
&\quad\geq \lr{\inf_z\inf_{\abs{b}\leq \eta}\rho_{b_z}(b)}^{k-1}\E{\vol_{\nin-1}\lr{P_\alpha\cap K}}.
\end{align*}
Summing this lower bound over $\alpha$ yields
\begin{align*}
&\E{\vol_{\nin-k}\lr{\twiddle{S}_{z_1,\ldots, z_k}\cap K}}\\
&\quad\geq \lr{\inf_z\inf_{\abs{b}\leq \eta}\rho_{b_z}(b)}^{k-1}\E{\vol_{\nin-1} \lr{\twiddle{S}_{z_1}\cap K}}.
\end{align*}
Applying the inductive hypothesis once more completes the proof. \hfill$\square$

\section{Proof of Corollary \ref{C:dist-formal}}\label{S:dist-pf}
We will need the following observation. 
\begin{lemma}\label{L:tube-vol}
  Fix a positive integer $n\geq 1$, and let $S\subseteq \R^n$ be a compact continuous piecewise linear submanifold with finitely many pieces. Define $S_0=\emptyset$ and let $S_k$ be the union of the interiors of all $k$-dimensional pieces of $S\backslash (S_0\cup\cdots\cup S_{k-1})$. Denote by $T_\ep(X)$ the $\ep-$tubular neighborhood of any $X\subset \R^n.$ We have
\[\vol_{n}\lr{T_{\ep}(S)}~\leq~ \sum_{k=0}^n \w_{n-k} \ep^{n-k}\vol_k\lr{S_k},\]
where $ \w_{d}:=\text{volume of ball of radius }1\text{ in }\R^d.$
\end{lemma}
\begin{proof}
Define $d$ to be the maximal dimension of the linear pieces in $S.$ Let $x\in T_{\ep}(S).$ Suppose $x\not \in T_{\ep}(S_k)$ for all $k=0,\ldots, d-1.$ Then the intersection of the ball of radius $\ep$ around $s$ with $S$ is a ball inside $S_d\cong U \subset \R^d$. Using the convexity of this ball, there exists a point $y$ in $S_d$ so that the vector $x-y$ is parallel to the normal vector to $S_d$ at $y$. Hence, $x$ belong to the normal $\ep$-ball bundle $B_\ep(N^*(S_d))$ (i.e.~the union of the fiber-wise $\ep$-balls in the normal bundle to $S_d$). Therefore, we have
\[\vol_n\lr{T_\ep(S)}\leq \vol_n(B_\ep(N^*(S_d)))+ \vol_n\lr{T_\ep(S_{\leq d-1})},\]
where we abbreviated $S_{\leq d-1}:=\overline{\bigcup_{k=0}^{d-1} S_k}.$ Using that
\begin{align*}
 \vol_n(B_\ep(N^*(S_d)))&=\vol_{d}(S_d)\vol_{n-d}(B_\ep(\R^{n-d}))\\
&= \vol_{d}(S_d)\ep^{n-d}\w_{n-d}
\end{align*}
and repeating this argument $d-1$ times completes the proof. 
\end{proof}
We are now ready to prove Corollary \ref{C:dist}. Let $x\in K=[0,1]^{\nin}$ be uniformly chosen. Then, for any $\ep>0$, using Markov's inequality and Lemma \ref{L:tube-vol}, we have
\begin{align*}
&\E{\mathrm{distance}(x,\mB_{\mN})}\\
&\quad\geq ~ \ep \P\lr{\mathrm{distance}(x,\mB_{\mN})>\ep}\\
&\quad= ~ \ep\lr{1- \P\lr{\mathrm{distance}(x,\mB_{\mN})\leq \ep}}\\
&\quad= ~ \ep\lr{1- \E{\vol_{\nin}\lr{T_\ep\lr{\mB_{\mN}}}}}\\
&\quad\geq  ~ \ep\lr{1- \sum_{k=1}^{\nin}\w_{\nin-k}\ep^{n_{in-k}}\E{\vol_{n_{in-k}}(\mB_{\mN,k})}}\\
&\quad\geq  ~ \ep\lr{1- \sum_{k=1}^{\nin}(C_{\mathrm{grad}}C_{\mathrm{bias}}\ep \#\set{\mathrm{neurons}})^k}\\
&\quad\geq  ~ \ep\lr{1- C'C_{\mathrm{grad}}C_{\mathrm{bias}}\ep\#\set{\mathrm{neurons}}}
\end{align*}
for some $C'>0.$ Taking $\ep$ to be a small constant times $1/(C_{\mathrm{grad}}\#\set{\mathrm{neurons}})$ completes the proof. \hfill $\square$
\end{document}